%% file: main.tex
\documentclass{article}

\usepackage[final]{neurips_2025}

\usepackage[utf8]{inputenc} 
\usepackage[T1]{fontenc}    
\usepackage{hyperref}       
\usepackage{url}            
\usepackage{booktabs}       
\usepackage{amsfonts}       
\usepackage{nicefrac}       
\usepackage{microtype}      
\usepackage[dvipsnames]{xcolor}         

\usepackage{graphicx} 

\usepackage{mkolar_definitions}
\usepackage[normalem]{ulem}
\usepackage{xspace}

\usepackage{algorithm, algorithmic}

\usepackage{enumitem}

\usepackage{titlesec}

\titlespacing*{\section}
{0pt}      
{6pt}      
{4pt}      

\titlespacing*{\subsection}
{0pt}{4pt}{2pt}

\DeclareMathOperator\Onest{OnEst}
\newcommand{\Onstate}{\Onest^{\mathrm{State}}}
\newcommand{\Ontraj}{\Onest^{\mathrm{Traj}}}

\newcommand{\Ostate}{\Ocal^{\mathrm{State}}}
\newcommand{\Otraj}{\Ocal^{\mathrm{Traj}}}

\newtheorem{assumption}{Assumption}
\newtheorem{observation}{Observation}

\newcommand{\Doff}{D_{\mathrm{off}}}
\newcommand{\Noff}{N_{\mathrm{off}}}
\newcommand{\Non}{N_{\mathrm{int}}}

\newcommand{\StateDagger}{\ensuremath{\textsc{Stagger}}\xspace}
\newcommand{\TrajDagger}{\ensuremath{\textsc{Tragger}}\xspace}
\newcommand{\WarmStartDagger}{\ensuremath{\textsc{Warm-Stagger}}\xspace}
\newcommand{\WarmStartTragger}{\ensuremath{\textsc{Warm-Tragger}}\xspace}
\newcommand{\Bbc}{\Bcal_{\text{bc}}}

\DeclareMathOperator\MDP{\Mcal}

\newcommand{\pie}{\pi^{\mathrm{E}}}
\newcommand{\pidet}{\Bcal^{\mathrm{Det}}}

\renewcommand{\paragraph}[1]{\noindent\textbf{#1}}

\newcommand{\alg}{\mathbb{A}}

\newcommand{\eclass}{\Bcal^\mathrm{E}}

\newcommand{\bE}{\mathbf{E}}
\newcommand{\bob}{\mathbf{b}}

\title{Interactive and Hybrid Imitation Learning: Provably Beating Behavior Cloning}

\author{%
  Yichen Li \\
  University of Arizona\\
  \texttt{yichenl@arizona.edu}\\
  \And
  Chicheng Zhang\\
  University of Arizona\\
  \texttt{chichengz@cs.arizona.edu}\\
}

\begin{document}

\maketitle

\begin{abstract}
Imitation learning (IL) is a paradigm for learning sequential decision-making policies from experts, leveraging offline demonstrations, interactive annotations, or both. Recent advances show that when annotation cost is tallied per trajectory, Behavior Cloning (BC)—which relies solely on offline demonstrations—cannot be improved in general, leaving limited conditions for interactive methods such as DAgger to help. We revisit this conclusion and prove that when the annotation cost is measured per state, algorithms using interactive annotations can provably outperform BC.
Specifically: (1) we show that \StateDagger, a one‑sample‑per‑round variant of DAgger, provably beats BC  under low-recovery-cost settings; (2) 
we initiate the study of hybrid IL where the agent learns from offline demonstrations and interactive annotations. We propose \WarmStartDagger whose learning guarantee is not much worse than using either data source alone. Furthermore, motivated by compounding error and cold‑start problem in imitation learning practice, we give an MDP example in which \WarmStartDagger has significant better annotation cost; (3) experiments on MuJoCo continuous‑control tasks confirm that, with modest cost ratio between interactive  and offline annotations, interactive and hybrid approaches consistently outperform BC. To the best of our knowledge, our work is the first to highlight the benefit of state‑wise interactive annotation and hybrid feedback in imitation learning.
\end{abstract}


\section{Introduction}
\input{intro}



\section{Preliminaries}
\label{sec:prelims}


\paragraph{Basic notation.} Define $[n] := \cbr{1,\ldots, n}$. Denote by $\Delta(\Xcal)$ the set of probability distributions over a set $\Xcal$. For $u \in \Delta(\Xcal)$ and $x \in \Xcal$, we denote by $u(x)$ the $x$-th coordinate of $u$ and $e_x$ the delta mass on $x$. We use the shorthand $x_{1:n}$ to represent the sequence 
$(x_i)_{i=1}^n$; 
we will also apply this shorthand to tuples, e.g. using $(x,y)_{1:n}$ to denote $(x_i, y_i)_{i=1}^n$.  
We will frequently use the Hellinger distance to measure the difference between two distributions:
$D_H^2(\mathbb{P},\mathbb{Q}) = \int( \sqrt{\tfrac{d\mathbb{P}}{d\omega}} - \sqrt{\tfrac{d\mathbb{Q}}{d\omega}} )^2 d\omega$, where $\mathbb{P}$ and $\mathbb{Q}$
share a dominating measure $\omega$.



\paragraph{Episodic Markov decision process and agent-environment interaction.}
A fixed-horizon episodic MDP $\Mcal$ is defined as a tuple $(\mathcal{S}, \mathcal{A}, P, \rho, \mathcal{R}, H)$, where $\mathcal{S}$ is the state space, $\mathcal{A}$ is the action space, 
$P: \Scal \times \Acal \rightarrow \Delta(\Scal)$
denotes the transition dynamics,
$\rho \in \Delta(\Scal)$ is the initial state distribution,
$\mathcal{R}: \Scal\times \Acal \rightarrow \Delta([0,1])$ denotes the reward distribution, and $H$ denotes episode length. 
\footnote{Here we assume that the transition dynamics and reward functions are stationary, i.e., it does not depend on time step in the episode. To translate our results to the  nonstationary transition dynamics and reward setting (to have a fair comparison with~\cite[e.g.,][]{foster2024behavior,rajaraman2020toward}), we can augment the state with a step index, i.e. define $\tilde{s} = (s, h)$. \label{fn:can-capture-layered-setting} 
}
Given a stationary policy $\pi : \Scal \rightarrow \Delta(\Acal)$, we use $\pi(\cdot |s)$ to denote the action distribution of $\pi$ on state $s$.
Rolling out policy $\pi$ in MDP $\Mcal$ gives a distribution over trajectories 
$\tau = (s_h, a_h, r_h)_{h=1}^H$ by first drawing the initial state $s_1 \sim \rho$, and then iteratively taking actions $a_h \sim \pi(\cdot \mid s_h)$, receiving rewards $r_h \sim \mathcal{R}(s_h, a_h)$, and transitioning to the next state $s_{h+1} \sim P(s_h, a_h)$ (except at step $H$). Let $\EE^{\pi}$ and $\mathbb{P}^\pi$ 
denote expectation and probability law for $(s_h, a_h, r_h)_{h=1}^H$ induced by $\pi$ and $\Mcal$. Given $\pi$, denote by
$d^\pi(s) := \frac{1}{H}
\sum_{h=1}^H
\mathbb{P}^\pi(s_h = s)$
its state visitation distribution.
The expected return of policy $\pi$ is defined as
$
J(\pi) := \mathbb{E}^\pi \left[ \sum_{h=1}^H r_h \right],
$ and the value functions of $\pi$ are given by
$V_h^\pi(s) := \mathbb{E}^\pi \left[ \sum_{h'=h}^H r_{h'} \,|\, s_h = s \right]$, and $Q_h^\pi(s, a) := \mathbb{E}^\pi \left[ \sum_{h'=h}^H r_{h'} \,|\, s_h = s, a_h = a \right].$
If for policy $\pi$, step $h$, and state $s$, $\pi(\cdot \mid s)$ is the delta-mass on an action, we also sometimes slightly abuse the notation and let $\pi(s)$ denote that action.



\paragraph{Additional policy-related notations.}  Throughout, we assume the access to a class $\Bcal$ of stationary policies
of finite size $B$.
A $(\text{MDP}, \text{Expert})$ pair $(\Mcal,\pie)$ is said to be $\mu$-recoverable if for all $h \in [H]$, $s\in \Scal$ and $a \in \Acal$, $Q_h^{\pi^{\mathrm{E}}}(s,a) - V_h^{\pi^{\mathrm{E}}}(s) \leq \mu$. 
Additionally, we assume normalized return~\cite{foster2024behavior}, where there exists some $R > 0$ such that for any trajectory $(s_h, a_h, r_h)_{h=1}^H$ rolled out in $\Mcal$, $\sum_{h =1}^H r_h \in [0,R]$. It always holds that $\mu \leq R$, while in many applications we expect $\mu$ to be much smaller~\cite[][Section 2.2]{ross2011reduction}. 
Throughout this paper, we make the assumption that our imitation learning problem is \emph{deterministically realizable}: 
\begin{assumption}
[Deterministic Realizability]
\label{assum:realizability}
The expert policy $\pie$ is deterministic and is contained in the learner's policy class $\Bcal$. 
\end{assumption}





In our algorithm and analysis, we frequently use the following ``convexification'' of policy class $\Bcal$:
\begin{definition}[Each-step mixing and each-step mixing policy class]
\label{def:mixed_class_each_step}
For $u$ in $\Delta(\Bcal)$, define its each-step mixing policy $\bar{\pi}_u$ as a stationary policy from $\Scal$ to $\Delta(\Acal)$: 
\[
\bar{\pi}_u(a \mid s) := \sum_{\pi \in \Bcal} u(\pi) \pi(a \mid s).
\]
Define $\bar{\Pi}_{\Bcal} := \cbr{ \bar{\pi}_u: u \in \Delta(\Bcal) }$ as $\Bcal$'s each-step mixing policy class.
\end{definition} 

An each-step mixing policy $\bar{\pi}_u \in \bar{\Pi}_{\Bcal}$
can be executed by drawing
 $\pi \sim u$ freshly-at-random at each step $h \in [H]$ and taking action $a_h \sim \pi(\cdot | s_h)$ (e.g.~\cite{li2022efficient,li2023agnostic}). 

\paragraph{Offline imitation learning and Behavior Cloning.} In offline imitation learning, the agent is given a collection of expert trajectories $\Dcal = \cbr{ \tau_1, \ldots, \tau_{\Noff} }$,
where $\tau_i = (s_{i,h}, a_{i,h})_{h=1}^H$ is the $i$-th (reward-free) trajectory, all of which are drawn iid from the trajectory distribution of expert policy $\pie$. Behavior Cloning finds policy $\pi \in \Bcal$ that minimizes its log loss on expert's actions on the seen states:
\[
\hat{\pi} = 
\argmin_{\pi \in \Bcal}
\sum_{i=1}^{\Noff} \sum_{h = 1}^H 
\log \frac{1}{\pi(a_{i,h} \mid  s_{i,h})}.
\]
Recent work of~\cite{foster2024behavior} establishes a sharp horizon-independent sample complexity bound of Behavior Cloning, which we recall here: 
\begin{theorem}[\cite{foster2024behavior}]
\label{thm:bc-guarantee}
Suppose Assumption~\ref{assum:realizability} holds, then with probability $1-\delta$, the policy returned by BC $\hat{\pi}$ satisfies: 
\[
J(\pie) - J(\hat{\pi})
\leq 
\tilde{O}\rbr{ \frac{ R \log B }{\Noff} }.
\]
\end{theorem}


\paragraph{Interactive imitation learning protocol.} In interactive IL, the learner has the ability to query the demonstration expert interactively. A first way to model interaction with expert is through a \emph{trajectory-wise demonstration oracle $\Otraj$}~\cite{ross2011reduction,foster2024behavior}: given a state sequence $(s_h)_{h=1}^H$, return $(a_h)_{h=1}^H$ such that $a_h = \pie(s_h)$ for all $h$. Subsequent works have considered modeling the interaction with expert as interacting with a \emph{state-wise demonstration oracle}~\cite{judah2014active,brantley2019disagreement,nguyen2020active,sekhari2024selective} $\Ostate$: given a state $s_h$ and step $h$, return $a_h = \pie(s_h)$.
We consider the learner interacting with the environment and demonstration oracles using the following protocol: 

\begin{itemize}[itemsep=0pt,topsep=0pt,partopsep=0pt]
\renewcommand\labelitemi{}
            \item \textbf{For} $i = 1, 2, \dots$

            \begin{itemize}
                \item Select policy $\pi^i$ and roll it out in $\Mcal$, observing a reward-free trajectory
                $(s_1, a_1, \ldots, s_H, a_H)$.
                \item Query the available oracle(s) to obtain expert annotations.
            \end{itemize}
            \item \textbf{Goal:} Return policy $\hat{\pi}$ such that $J(\pie) -J(\hat{\pi}) $ is small, with a few number of queries to $\Otraj$ or $\Ostate$.
        \end{itemize}

In practice, we expect the cost of querying $\Otraj$ to be higher than that of collecting a single offline expert trajectory~\cite{laskey2017comparing}. Since $H$ queries to $\Ostate$ can simulate one query to $\Otraj$, the cost of a single $\Ostate$ query should be at least $\frac{1}{H}$ the cost of $\Otraj$. Consequently, we also expect one $\Ostate$ query to be more expensive than obtaining an additional offline (state, expert action) pair. We denote the ratio between these two costs as $C$, where $C \geq 1$ is an application-dependent constant. 
\footnote{For practical settings such as human-in-the-loop learning with expert interventions~\cite{mandlekar2020human,spencer2020learning}, obtaining a short segment of corrective demonstrations may be  
cheaper than querying $\Ostate$ for each state therein. Thus, our cost model may need refinements to match practical applications, which we leave as interesting future work.}

\section{State-wise Annotation in Interactive Imitation Learning}


Recent work~\cite{foster2024behavior} on refined analysis of Behavior Cloning (BC) casts doubt on the utility of interaction in imitation learning: when measuring sample complexity in the number of trajectories annotated, BC is minimax optimal even among interactive algorithms~\cite[][Corollary 2.1 and Theorem 2.2]{foster2024behavior}. Although benefits of interactive approaches have been shown in specific examples, progresses so far have been sparse~\cite{foster2024behavior,rajaraman2021value}, with the most general results in the less-practical tabular setting~\cite{rajaraman2021value}. 
In this section, we reexamine this conclusion and show that interaction benefits imitation learning in a fairly general sense in the general function approximation setting: when measuring sample complexity using the number of state-wise annotations, we design an interactive algorithm with sample complexity better than BC, as long as the expert policy has a low recovering cost $\mu$ in the environment.

\begin{algorithm}[t]
\caption{$\StateDagger$: DAgger with State-wise annotation oracle}
\label{alg:dagger_one_sample}
\begin{algorithmic}[1]
\STATE \textbf{Input:} MDP $\Mcal$, state-wise expert annotation oracle $\Ostate$ with query budget $\Non$, stationary policy class $\Bcal$, online learning oracle $\alg$.
%
\FOR{$n = 1, \dots, \Non$}
    \STATE Query $\alg$ and receive $\pi^n$.
    \STATE Execute $\pi^n$ 
    and sample $s^n \sim d^{\pi^n}$. Query $\Ostate$ for $a^{*,n}=\pie(s^n)$.
    \STATE Update $\alg$ with loss function
    
    \begin{equation}
    \ell^n(\pi)
    := \log\rbr{\frac1{
    \pi(a^{*,n}|s^n)}}
    .
    \label{eqn:ell-n-u-1}
    \end{equation}
\ENDFOR
\STATE Output $\hat{\pi}$, a first-step uniform mixture of $\cbr{\pi^n}_{n=1}^{\Non}$.
\end{algorithmic}
\end{algorithm}

\subsection{Interactive IL Enables Improved Sample Complexity with State-wise Annotations}

Our algorithm \StateDagger (short for State-wise DAgger), namely Algorithm~\ref{alg:dagger_one_sample}, interacts with the demonstration expert using a state-wise annotation oracle $\Ostate$. 
Similar to the original DAgger~\cite{ross2011reduction}, it requires base policy class $\Bcal$ and reduces interactive imitation learning to 
no-regret online learning. 
At iteration $n$, 
it rolls out the current policy $\pi^n$
obtained from an online learning oracle 
$\alg$ and samples state $s^n$ from $d^{\pi^n}$. A classical example of $\alg$ is the exponential weight algorithm that chooses policies from $\bar{\Pi}_{\Bcal}$ (\cite{cesa2006prediction}; see also Proposition~\ref{prop:log_loss_base} in Appendix~\ref{sec:auxiliary_results}).
It then queries $\Ostate$ to get expert action $a^{*,n}$ and updates $\alg$ with loss function $\ell^n(\pi)$ induced by this new example (Eq.~\eqref{eqn:ell-n-u-1}). 
The final policy $\hat{\pi}$ is returned as a uniform first-step mixture of the historical policies $\{\pi^n\}_{n=1}^{\Non}$, i.e., sample one $\pi^n$ uniformly at random and execute it for the episode.
In contrast to the DAgger variants analyzed in~\cite{foster2024behavior,rajaraman2021value}
, which trains a distinct policy at each step—yielding $H$ policies in total—and employs trajectory-level annotations, our algorithm utilizes parameter sharing of the policy representation across all steps and uses state-wise annotations.\footnote{A one-sample-per-iteration version of AggreVate~\cite{ross2014reinforcement} has been analyzed in~\cite[][Section 15.5]{agarwal2019reinforcement}. Our analysis follows a similar structure and further  takes direct  advantage of the expert action feedback in DAgger and the deterministic realizability assumption to get refined sample complexity.}

We show the following performance guarantee of Algorithm~\ref{alg:dagger_one_sample}
with $\alg$ instantiated as the exponential weight algorithm: 






\begin{theorem}
\label{thm:log_loss_dagger_main}
Suppose \StateDagger is run with a state-wise expert annotation oracle $\Ostate$, an MDP $\Mcal$ where $(\Mcal, \pie)$ is $\mu$-recoverable, a policy class $\Bcal$ such that deterministic realizability (Assumption~\ref{assum:realizability}) holds, and the online learning oracle $\alg$ set as the exponential weight algorithm
with decision space 
$\bar{\Pi}_{\Bcal}$. Then it returns $\hat{\pi}$ such that, with probability at least $1-\delta$,
\[
J(\pie) - J(\hat{\pi})
\leq 
\mu H \cdot \frac{ \log(B) + 2\log(1/\delta)}{\Non}.
\]
\end{theorem}

Theorem~\ref{thm:log_loss_dagger_main} shows that \StateDagger returns a policy of suboptimality $O(\frac{\mu H \log B}{\Non})$ using $\Non$ interactive state-wise annotations from the expert. 
In comparison, with the cost of $\Non$ state-wise annotations, one can obtain $\frac{C\Non}{H}$ trajectory-wise annotations;~\cite{foster2024behavior}'s analysis shows that Behavior Cloning with this number of trajectories from $\pie$ returns a policy of suboptimality $O(\frac{R H \log B}{C \Non})$ (recall Theorem~\ref{thm:bc-guarantee}). 
Thus, if $C \ll \frac{R}{\mu}$, Algorithm~\ref{alg:dagger_one_sample} has a better cost-efficiency guarantee than Behavior Cloning.


We now sketch the proof of Theorem~\ref{thm:log_loss_dagger_main}. 
In line with~\cite{foster2024behavior}, we define the online, on-policy state-wise estimation error as
\[
\Onstate_N := \sum_{n=1}^{N} \EE_{s \sim d^{\pi^n}} \sbr{D^2_\text{H}(\pi^n(\cdot \mid s),\pie(\cdot \mid s)) }.
\]

The proof proceeds by bounding this error and translating it to the performance difference between $\hat\pi$ and $\pie$. While our definition of estimation error is similar to~\cite[][Appendix C.2]{foster2024behavior}, their definition 
requires all $H$ states per trajectory, while ours depends on the distribution over a state sampled uniformly from the rollout of policy $\pi^n$. This enables each labeled state to serve as immediate online feedback, fully utilizing the adaptivity
of online learning. 



\subsection{Experimental Comparison}

We conduct a simple simulation study comparing the sample efficiency of log-loss Behavior Cloning~\cite{foster2024behavior} and \StateDagger in four MuJoCo~\cite{todorov2012mujoco, brockman2016openai} continuous control tasks with $H=1000$ and pretrained deterministic MLP experts~\cite{schulman2015trust,schulman2017proximal}. Considering MuJoCo's low sensitivity to horizon length~\cite{foster2024behavior}, we reveal expert states one by one along consecutive trajectories for BC to allow fine-grained state-wise sample complexity
comparison, while \StateDagger queries exactly one state per iteration by sampling from the latest policy's rollout and updating immediately with the expert's annotation. 
In \StateDagger, we implement the online learning oracle $\alg$ so that it outputs a policy that approximately minimizes the log loss.
In addition to log loss, we also include results with online learning oracle minimizing historical examples' total square loss 
in Appendix~\ref{sec:mse_experiment}.
We defer other implementation details to Appendix~\ref{sec:appendix_experiment_results}. 


Figure~\ref{fig:state_wise_sample} shows the performance of the learned policy as a function of the number of state-wise annotations. When each interactive state-wise annotation has the same cost as an offline (state, expert action) pair ($C=1$),  \StateDagger has superior and more stable performance than Behavior Cloning. For a given target performance (e.g., near expert-level), \StateDagger often requires significantly fewer state-wise annotations than BC—especially on harder tasks—though the gains are less pronounced on easier ones like Ant and Hopper. To highlight sample efficiency, we plot \StateDagger using only half the annotation budget of BC; despite this, it still matches or surpasses BC on several tasks, suggesting meaningful benefits from interaction when $C$ is small (e.g., $C = 3$ for Walker).




\begin{figure}[t]  
  \centering  \includegraphics[width=\textwidth]{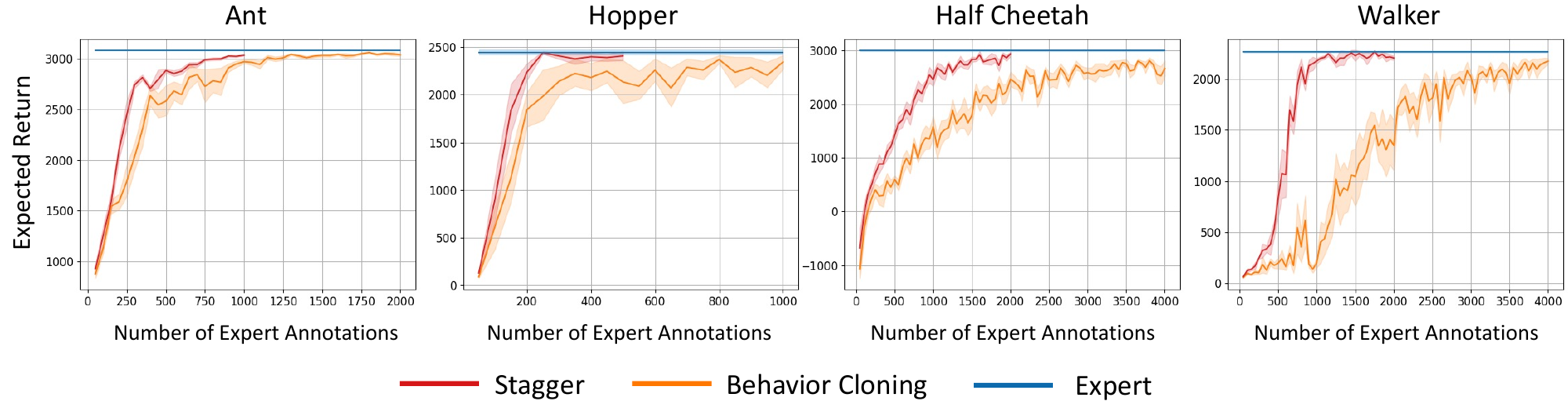}
  \caption{
State-wise sample complexity comparison between Behavior Cloning and \StateDagger. 
Shaded areas show the 10th–90th percentile bootstrap confidence intervals~\cite{diciccio1996bootstrap} over 10 runs. 
\StateDagger matches or exceeds BC with $50\%$ fewer annotations, achieving better state-wise annotation efficiency.
}
\label{fig:state_wise_sample}
\end{figure}

\section{Hybrid Imitation Learning: Combining Offline Trajectory-wise and Interactive State-wise Annotations}

Practical deployments of imitation learning systems often learn simultaneously from offline and interactive feedback modalities~\cite{kelly2019hg, hoque2021thriftydagger}: for example, in autonomous driving~\cite{zhang2017query, badue2021self, zhao2023real}, the learner has access to some offline expert demonstrations to start with, and also receives interactive expert demonstration feedback in trajectory segments for subsequent finetuning. 
Motivated by this practice, we formulate the following problem:

\paragraph{Hybrid Imitation Learning (HyIL): Problem Setup.}
The learner has access to two complementary sources of expert supervision:
\begin{itemize}
\item $\Noff$ offline expert trajectories $\Doff = \cbr{ (s_{i,h}, a_{i,h})_{h=1}^H, i \in [\Noff] }$, sampled i.i.d.\ from rolling out $\pi^E$ in $\Mcal$;
\item A state-wise annotation oracle $\Ostate$ that can be queried interactively up to $\Non$ times.
\end{itemize}
Each offline (state, action) pair takes a unit cost, and the cost of an interactive query is  $C \geq 1$. The total cost budget is therefore $H \cdot \Noff + C \cdot \Non$.
The goal is to return a policy $\hat{\pi}$ that minimizes its suboptimality relative to the expert policy $J(\pi^E) - J(\hat{\pi})$. 


We ask: can we design a HyIL algorithm with provable sample efficiency guarantee? Furthermore, can its performance surpass pure BC and pure interactive IL with the same cost budget?

\begin{algorithm}[t]
\caption{\WarmStartDagger: Warm-start \StateDagger with offline demonstrations}
\label{alg:warm_il}
\begin{algorithmic}[1]
\STATE \textbf{Input:} MDP $\Mcal$, state-wise expert annotation oracle $\Ostate$, stationary policy class $\Bcal$, online learning oracle $\alg$, offline expert dataset $D_{\text{off}}$ of size $\Noff$, online budget $\Non$
\STATE Initialize $\alg$ with policy class $\Bcal_{\text{bc}} := \{ \pi \in \Bcal : \pi(s_h) = a_h,\ \forall h \in [H], \forall (s, a)_{1:H} \in D_{\text{off}} \}$.
\label{step:bc-class}
\FOR{$n = 1, \dots, \Non$}
    \STATE Query $\alg$ and receive $\pi^n$.
    \STATE Execute $\pi^n$ and sample
    $s^n \sim d^{\pi^n}$. Query $\Ostate$ for $a^{*,n}=\pie(s^n)$.
    \STATE Update $\alg$ with loss function:
    \begin{equation}
        \ell^n(\pi) := \log \left( \frac{1}{ 
        \pi(a^{*,n} \mid s^n)} \right).
        \label{eqn:loss-warmstart}
    \end{equation}
\ENDFOR
\STATE \textbf{Output:} $\hat{\pi}$, a first-step uniform mixture of $\{\pi^1, \dots, \pi^N\}$.
\end{algorithmic}
\end{algorithm}








\subsection{\WarmStartDagger: Algorithm and Analysis}

We answer the above questions by proposing the \WarmStartDagger algorithm, namely Algorithm~\ref{alg:warm_il}. 
It extends \StateDagger to incorporate offline expert demonstrations, in that it constructs $\Bbc$, a restricted policy class that contains all policies in $\Bcal$
consistent with all offline expert demonstrations (line~\ref{step:bc-class}). 
It subsequently performs online log-loss optimization on $\Bbc$ over state-action pairs collected online, where the state $s^n$ is obtained by rolling out $\pi^n$ in the MDP $\Mcal$, and the action $a^{*,n}$ is annotated by the state-wise expert annotation oracle $\Ostate$. For analysis, we introduce the following definitions.




\begin{definition}[Non-stationary Markovian policies]
A non-stationary Markovian policy $\nu = (\nu_1, \ldots, \nu_H)$ is a collection of $H$ mappings, with each $\nu_h$ in $\Delta(\Acal)^{\Scal}$, where upon rolling out $\nu$, at every step $h \in [H]$, the agent takes action $a_h$ by sampling from $\nu_h(\cdot \mid s_h)$.  
\end{definition}

\begin{definition}[Step-wise completion of stationary policy class] 
\label{def:policy_completion}
For a stationary policy class $\Bcal \subseteq \Delta(\Acal)^\Scal$, define its step-wise completion to be a class of nonstationary Markovian policies: 
\[
\tilde{\Bcal} = \cbr{ \nu = (\nu_1, \ldots, \nu_H): \nu_h \in \Bcal,  
\text{ for all } h \in [H]
}
\]
\end{definition}




In words, each \(\pi \in \tilde{\Bcal}\) uses a possibly distinct policy \(\pi_h\) from \(\Bcal\) to take action at step \(h\).
By definition, $\tilde{B} := |\tilde{\Bcal}|$ 
is at most \(B^H\).
An interesting special case is the \emph{non-parameter sharing setting}~\cite{ross2010efficient,rajaraman2020toward,rajaraman2021value,foster2024behavior}, where
the set of possible states visited at different steps are disjoint (see also footnote~\ref{fn:can-capture-layered-setting} for examples), and $\Bcal$ are \emph{factorized}, in the sense that parameters associated with the policies to use at different steps are separate. In this case, we have that $\tilde{\Bcal} = \Bcal$ and thus $\tilde{B} = B$. 



\begin{theorem}
\label{thm:wstagger_main}
If \WarmStartDagger is run with a state-wise expert annotation oracle $\Ostate$, an MDP $\Mcal$ where $(\Mcal, \pie)$ is $\mu$-recoverable, a policy class $\Bcal$ such that deterministic realizability (Assumption~\ref{assum:realizability}) holds, and the online learning oracle $\alg$ set as the exponential weight algorithm with each-step mixing policies, then it returns $\hat{\pi}$ such that, with probability at least $1-\delta$,
\begin{equation}
J(\pie) - J(\hat{\pi}) \leq O\left( \min \left( \frac{R \log(\tilde{B}/\delta)}{\Noff}, \frac{\mu H \log(B_{\text{bc}}/\delta)}{ \Non } \right) \right),
\label{eqn:warm-stagger-main}
\end{equation}
where we recall that $B \leq \tilde{B} \leq B^H$, and $B_{\text{bc}}:= |\Bcal_{\text{bc}}| \leq B$. 
\end{theorem}
Theorem~\ref{thm:wstagger_main} shows that \WarmStartDagger finds a policy with suboptimality guarantee not signficantly worse than BC or \StateDagger: first, Behavior Cloning using the offline data has a suboptimality guarantee of $O\rbr{ \frac{R \log(B/\delta)}{\Noff} }$ (cf. Theorem~\ref{thm:bc-guarantee}), and \WarmStartDagger's guarantee is worse by at most a factor of $H$; second, 
\StateDagger without using offline data has a suboptimality of $O \rbr{ \frac{\mu H \log(B/\delta)}{ \Non } }$ (cf. Theorem~\ref{thm:log_loss_dagger_main}), which is on par with the second term of Eq.~\eqref{eqn:warm-stagger-main}. 
We conjecture that the $\log\tilde{B}$ dependence may be sharpened to $\log B$, and leave this as an interesting open question.

\begin{remark}
One may consider another baseline that naively switches between BC and \StateDagger based on a comparison between their bounds; however, such a baseline needs to know $R$ and $\mu$ ahead of time. In practice, we expect \WarmStartDagger to perform much better than this baseline, since it seamlessly incorporates both sources of data, and its design does not rely on theoretical bounds that may well be pessimistic. 
\end{remark}


\subsection{On the Benefit of Hybrid Imitation Learning}
\label{sec:hyil_benefit_main}

Theorem~\ref{thm:wstagger_main} is perhaps best viewed as a fall-back guarantee for \WarmStartDagger: its performance is not much worse than using either of the feedback source alone.
In this section, we demonstrate that the benefit of hybrid feedback modalities can go beyond this: 
we construct an MDP, in which hybrid imitation learning has a significantly better sample efficiency than both offline BC and interactive \StateDagger. 
Specifically, we prove the following theorem: 



\begin{theorem}
\label{thm:tabular-ex}
For large enough $S,H,A$, 
there exists an episodic MDP $\Mcal$ with $S$ states, $A$ actions, and horizon $H$, and expert policy $\pie$ such that:
\begin{itemize}[itemsep=0pt]
\item With $\Omega(S)$ offline expert trajectories for BC, the learned policy is $\Omega(H)$-suboptimal;

\item With $\Omega(H S)$ interactive expert annotations for \StateDagger, the learned policy is $\Omega(H)$-suboptimal;

\item With $\tilde{O}\left( S/H \right)$ offline trajectories and $O(1)$ expert interactions, \WarmStartDagger learns a policy $\hat{\pi}$ such that $J(\hat{\pi}) = J(\pi^E)$.
\end{itemize}
\end{theorem}


Theorem~\ref{thm:tabular-ex} suggests that when $HS \gg \max(1, C)$ and $C \gg \frac{1}{H}$, \WarmStartDagger achieves expert-level performance with significantly lower cost than two baselines. To see this, observe that \WarmStartDagger has a total cost of $O(S + C)$, which is much smaller than $\Omega(HS)$ by BC, and $\Omega(HSC)$ by \StateDagger. 



\begin{figure*}[t]  
  \centering
  \includegraphics[width=\textwidth]{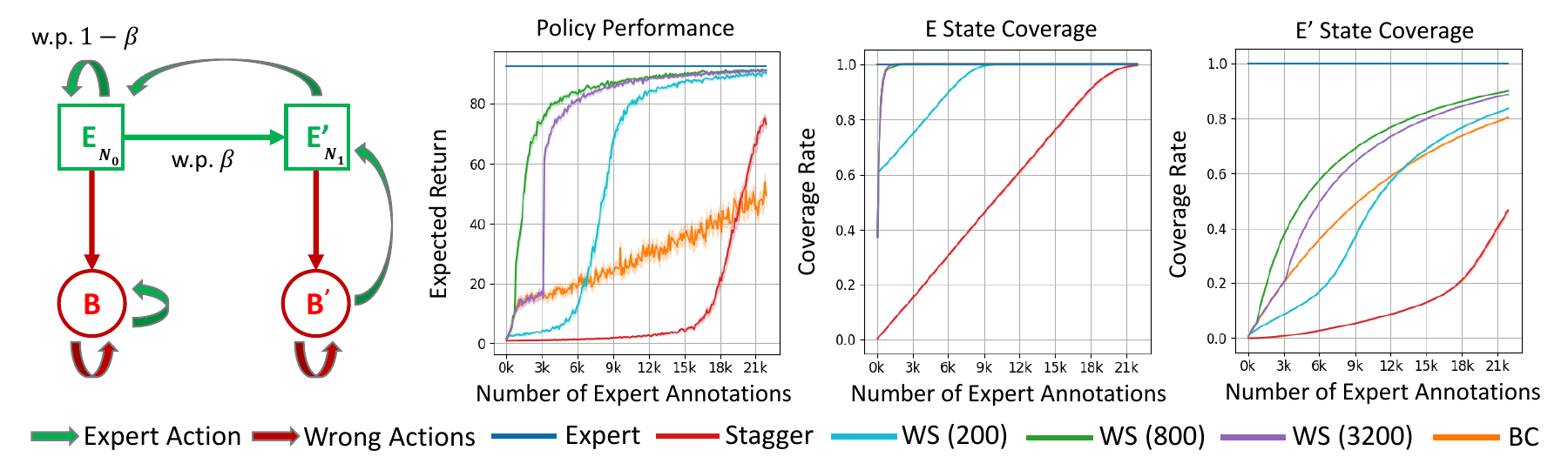}
  \caption{MDP construction and simulation results of algorithms with rewards assigned only in $\mathbf{E}$. We evaluate \WarmStartDagger (WS) with 
  200, 800, 3200 offline (state, expert action) pairs. All methods are evaluated under equal total annotation cost with $C = 1$.
  With 800 offline (state, expert action) pairs, WS significantly improves the sample efficiency over the baselines and explores $\mathbf{E'}$ more effectively.}
  \label{fig:mdp_design}
\end{figure*}


\paragraph{The MDP construction and simulation results.} We now sketch our construction of MDP $\Mcal$. $\Mcal$ has an episode length $H = \Omega(\log(S))$  and action space of size greater than $10H$.
For each state, one of the actions is taken by the expert; the rest are ``wrong'' actions.
We illustrate $\Mcal$'s state space on the left of Figure~\ref{fig:mdp_design};
specifically, it consists of the following subsets:
\begin{itemize}[itemsep=0pt]
\item Unrecoverable state $\mathbf{B} := \cbr{ \mathbf{b} }$: a special absorbing state that is unrecoverable by any action (dead).

\item Expert ideal states $\mathbf{E}$, where $|\mathbf{E}| = N_0$:
this can model for example, an autonomous driving agent driving stably on the edge of a cliff~\cite{ross2011reduction}, where any incorrect action transitions the agent to the unrecoverable state $\mathbf{b}$ (e.g., car falling off the cliff). Taking the expert action keeps the agent in $\mathbf{E}$ with high probability $(1 - \beta)$, and with a small probability $\beta$, moves the agent to $\mathbf{E'}$ (e.g., a safe slope).


\item Expert recoverable states $\mathbf{E'}$: this models the agent driving outside the edge of the cliff in a safe slope. When in $\mathbf{E'}$, taking the expert action allows the agent to return to a uniformly sampled state in $\mathbf{E}$. Taking a wrong action from $\mathbf{E'}$ leads to reaching a recoverable state $\mathbf{b'}$ (e.g., rest area).


\item Recoverable state $\mathbf{B'} = \cbr{ \mathbf{b'} }$: 
Not knowing how to act in $\mathbf{b'}$ will result in the agent getting trapped in $\mathbf{b'}$ for the episode.


\end{itemize}


We now briefly justify each algorithm’s learning performance as stated in Theorem~\ref{thm:tabular-ex}.
First, BC only observes expert actions in $\mathbf{E}$ and $\mathbf{E'}$, but never in $\mathbf{b'}$. As a result, near-expert performance at test time requires high coverage over $\mathbf{E'}$; otherwise, BC's trained policy will likely incur compounding errors and get trapped in $\mathbf{b}'$.
Second, \StateDagger suffers from a cold-start problem: early policies fail to explore $\mathbf{E}$ efficiently, and incorrect actions can cause transitions into $\mathbf{b}$. Consequently, coverage over $\mathbf{E}$ grows slowly, and the policy may still fail on unannotated states in $\mathbf{E}$ even with $\Omega(HS)$ queries.
Lastly, \WarmStartDagger benefits from offline data that fully covers $\mathbf{E}$, and uses a small number of interactions to visit $\mathbf{b'}$ and query the expert, avoiding costly exploration in $\mathbf{E'}$ while matching expert performance.

We also conduct a simulation of the aforementioned three algorithms in a variant of the above MDP with $N_0 = 200$, $N_1 = 1000$, $H = 100$, and $\beta = 0.08$, using another 
reward function that assigns a reward of 1 only when the agent visits the states in $\mathbf{E}$. Here, we let the online learning oracle $\alg$ optimize 0-1 loss on the data seen so far, which is equivalent to minimizing log loss under a class of deterministic policies and discrete actions. Figure~\ref{fig:mdp_design} shows return and state coverage as functions of the number of expert annotations, averaged over 200 runs. 


We observe that: (1) BC exhibits slow improvement, as $\mathbf{b'}$ remains unseen (and thus unannotated) throughout training, resulting in poor performance even with substantial coverage (e.g., $80\%$) over $\mathbf{E'}$; (2) \StateDagger is sample-inefficient due to slow exploration over states in $\mathbf{E}$, consistent with the cold-start intuition; (3)  \WarmStartDagger (WS), when initialized with limited (e.g., 200) offline (state, expert action) pairs, still needs to explore $\mathbf{E}$ first before it can safely reach $\mathbf{E'}$ without failure; and (4) \WarmStartDagger with sufficient offline coverage on $\mathbf{E}$ (e.g., initialized with 3200 offline (state, expert action) pairs) directly benefits from exploring $\mathbf{b'}$ with immediate performance gain, and enables safe and even faster exploration than the expert in $\mathbf{E'}$.

\begin{figure*}[t]  
  \centering
  \includegraphics[width=\textwidth]{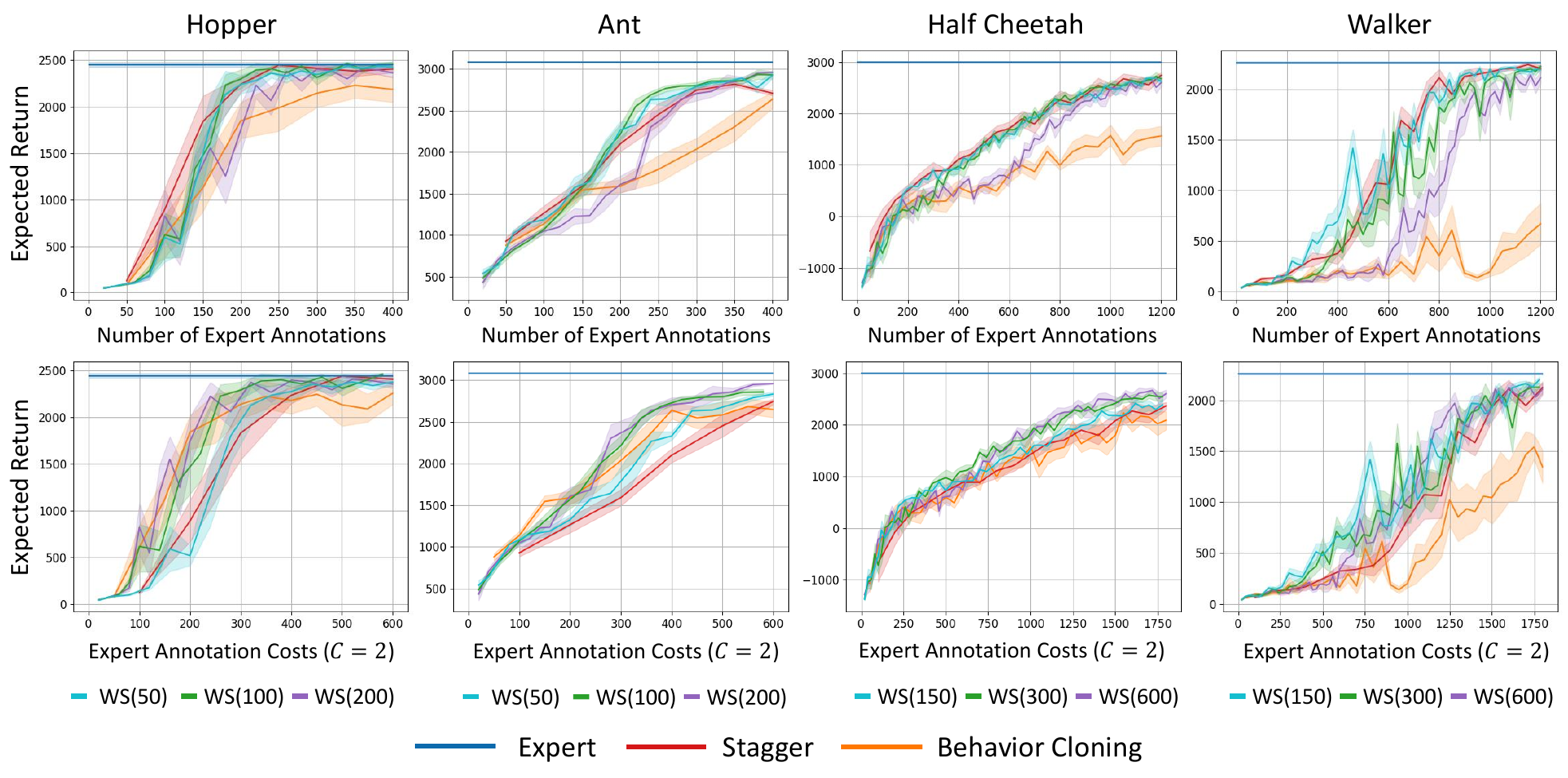}
  \caption{Sample and cost efficiency on MuJoCo tasks. The top row shows expected return vs. number of annotations ($C = 1$); the bottom row shows performance in a cost-aware setting ($C = 2$). \WarmStartDagger (WS) is initialized with 1/8, 1/4, or 1/2 of the total annotation budget as offline demonstrations. Specifically, WS($n$) refers to WS with offline expert trajectory demonstrations of total length $n$.
  For a good range of $n$'s, WS($n$) matches \StateDagger in sample efficiency and outperforms the baselines when $C = 2$.
  }
  \label{fig:log_comparison_main}
\end{figure*}


\subsection{Hybrid IL on Continuous Control Benchmarks}
\label{sec:experiments}
Following our earlier MuJoCo-based comparison of Behavior Cloning and \StateDagger, we now evaluate \WarmStartDagger (WS) on the same  benchmarks. This experiment aims to answer: Does WS reduce total annotation cost compared to the baselines?

Based on the observation in Figure~\ref{fig:state_wise_sample}, we assign 400 total state-wise annotations for Hopper and Ant, and 1200 for HalfCheetah and Walker2D. For \WarmStartDagger, we allocate 1/8, 1/4, or 1/2 of the total annotations to offline data, with the remainder used for interactive queries. For a fair comparison, all methods are evaluated under the same total annotation cost, with $C = 1$ or $C = 2$. This makes the baselines stronger, as they have full cost budget assigned to a single source.

In terms of the number of state-wise annotations ($C=1$), the results align with our theoretical findings: WS performs not significantly worse than BC or \StateDagger, regardless of the offline dataset size.
 WS still achieves performance competitive with \StateDagger, and even outperforms it on Ant when $C = 1$. 
Furthermore, as shown by the purple curves, WS with appropriate offline sample size has preferable performance over 4 tasks when $C = 2$, highlighting its utility in cost-aware regimes.
These results confirm that \WarmStartDagger reduces total annotation cost for moderate $C$.

\section{Related Work}
\input{literature}

\section{Conclusion}
We revisit imitation learning from the perspective of state-wise annotations. 
We show via the \StateDagger algorithm that, interaction with the demonstrating expert, with its cost properly measured, can enable provable cost efficiency gains over Behavior Cloning. 
We also propose \WarmStartDagger that combines the benefits of offline data and interactive feedback. Our theoretical analysis shows that such a hybrid method can strictly outperform both pure offline and pure interactive baselines under realistic cost models.
Empirical results on our synthetic MDP support our theoretical findings, while MuJoCo experiments demonstrate the practical viability and competitive performance of our methods on continuous control tasks. In Appendix~\ref{sec:decoupled_h_dist}, we also show that a trajectory-wise annotation variant of  DAgger can match the sample complexity of log-loss Behavior Cloning without recoverability assumptions, with additional experiments (Appendix~\ref{sec:tragger_experiments_appendix}).

\paragraph{Limitations:} Our design of imitation learning algorithm only aims at closing the gap between the performance of the expert and the trained policy; thus, the performance of our learned policy is bottlenecked by the expert's performance. In this respect, designing imitation learners that output policies surpassing expert performance is an important direction. 

Our theory provide sample complexity guarantees for the discrete-action setting with deterministic and  realizable expert. When such assumptions are relaxed, additional challenges arise~\cite{simchowitz2025pitfalls}. In this respect, there remains a gap between our theoretical analysis and our MuJoCo experiment results. In future work, we are interested in conducting additional experiments on discrete-action control problems (e.g., Atari) as well as language model distillation tasks.  

\paragraph{Acknowledgments:} We thank the anonymous NeurIPS reviewers for their helpful feedback, which significantly improved the presentation of the paper.
We thank Kiant\'{e} Brantley for helpful discussions on state-wise annotations. 
We thank National Science Foundation IIS-2440266 (CAREER) for research support. 


\bibliographystyle{plain}
\bibliography{il}

\newpage

\newpage
\input{appendix}

\end{document}

%% file: intro.tex
Imitation learning, also known as learning from demonstrations, is a widely-used approach 
for learning policies to make sequential decisions~\cite{pomerleau1988alvinn,bojarski2016end,bansal2018chauffeurnet}. 
In many applications, it offers a preferable alternative to reinforcement learning, as it bypasses the need for carefully designed reward functions and avoids costly exploration~\cite{osa2018algorithmic,sun2017deeply}.

Two prominent data collection regimes exist in imitation learning: offline and interactive. 
In offline imitation learning, expert demonstration data in the format of trajectories 
is collected ahead of time, which is a non-adaptive process that is easy to maintain.
In contrast, in interactive imitation learning, the learner is allowed to query the expert for annotations in an adaptive  manner~\cite{ross2011reduction,ross2014reinforcement,sun2017deeply}.
The most basic and well-known approach for offline imitation learning is Behavior Cloning~\cite{ross2010efficient,foster2024behavior}, which casts the policy learning problem as a supervised learning problem that learns to predict expert actions from states. 
Although simple and easy to implement, offline imitation learning has the drawback that the quality of the data can be limited~\cite{pomerleau1988alvinn}. 
As a result, the trained model can well suffer from compounding error, where imperfect imitation leads the learned policy to enter unseen states, resulting in a compounding sequence of mistakes.
In contrast, in interactive imitation learning, 
the learner maintains a learned policy over time, with the demonstrating 
experts providing  corrective feedback \emph{on-policy}, which enables targeted collection of  demonstrations and improves sample efficiency.


Recent work~\cite{foster2024behavior}, via a sharp analysis of Behavior Cloning, shows that the sample efficiency of Behavior Cloning cannot be improved in general when measuring using the number of trajectories annotated.
Interactive methods like DAgger~\cite{ross2010efficient} 
can enjoy sample complexity benefits,
but so far the benefits are only exhibited in limited examples, with the most general ones in the tabular setting~\cite{rajaraman2021value}. This leaves open the question:







\begin{quote}
\textit{Can interaction provide sample efficiency benefit for imitation learning under diverse settings, notably in the presence of function approximation?}
\end{quote}







In this paper, we make progress towards this question, with a focus on the \emph{deterministically realizable} setting (i.e. the expert policy $\pi^E$ is deterministic and is in the learner's policy class $\Bcal$). 
Specifically, we make the following contributions: 

\begin{enumerate}[itemsep=0pt,topsep=0pt,partopsep=0pt]
\item Motivated by the costly nature of interactive labeling on entire trajectories~\cite{laskey2017comparing,mandlekar2020human}, we propose to measure the cost of annotation 
using the number of states annotated by the demonstrating expert. 
We propose a general 
state-wise interactive imitation learning algorithm, \StateDagger, 
and show that as long as the expert can recover from mistakes at low cost in the environment~\cite{ross2011reduction}, 
it significantly improves over Behavior Cloning in terms of its number of state-wise demonstrations required. 



\item Motivated by practical imitation learning applications where sets of offline demonstration data are readily available, we study \emph{hybrid imitation learning}, where the learning agent can additionally query the demonstration expert interactively to improve its performance. 
We design a hybrid imitation learning algorithm, \WarmStartDagger, and prove that its policy optimality guarantee is not much worse than using either of the data sources alone. 

\item Inspired by compounding error~\cite{pomerleau1988alvinn} and cold start problem~\cite{lin2020autoregressive,oetomo2023cutting}, 
we provide an MDP example, 
for which we show hybrid imitation learning can achieve strict sample complexity savings over using either source alone, and provide simulations that verify this claim.






\item 
We conduct experiments 
in MuJoCo continuous control tasks
and show that if the cost of state-wise interactive demonstration is not much higher than its offline counterpart, interactive algorithms can enjoy a better cost efficiency than Behavior Cloning. 
Under some cost regimes and some environments, hybrid imitation learning can outperform approaches that use  either source alone.

\end{enumerate}

%% file: literature.tex

\paragraph{Imitation Learning with offline demonstrations}, pioneered in autonomous driving~\cite{pomerleau1988alvinn}, 
was solved by offline, state‑wise supervised learning in early works~\cite{ross2010efficient,syed2010reduction} and named Behavior Cloning (BC). A recent analysis by~\cite{foster2024behavior} employs trajectory-wise Hellinger distance to tighten the dependence of BC on the horizon at the trajectory level, although its sample complexity measured per state still grows quadratically with the horizon in the worst case.
This shortcoming, often termed
covariate shift or compounding error~\cite{pomerleau1988alvinn}, arises when the learned policy's imperfect imitation drives the learner to unseen states, resulting in a cascading sequence of mistakes.
From a data collection perspective, this can be mitigated by noise‑injection approaches such as~\cite{laskey2017dart,ke2021grasping}. By leveraging additional environment interactions,
generative‑adversarial IL methods~\cite{ho2016generative,sun2019provably,ke2020imitation,spencer2021feedback} 
frame learning as a two‑player game, and aims to find a policy that matches expert’s state‑action visitation distributions. This setting is 
also known as  ``apprenticeship learning using inverse reinforcement learning'' in earlier works~\cite{abbeel2004apprenticeship,syed2007game}.
Quantitative comparisons with these methods are beyond our scope, as they rely on extensive interactions with the MDP and access to a class of discriminator functions, while we focus on understanding the utility of state‑wise interactive annotations. This line of works also include recent work of~\cite{ren2024hybrid}, who introduce ``Hybrid Inverse Reinforcement Learning'', which leverages hybrid reinforcement learning~\cite{songhybrid} to accelerate its inner loop of policy search; different from theirs, our ``hybrid'' setting focus on utilizing heterogeneous data modalities.
Recent offline imitation learning approaches~\cite{chang2021mitigating,zeng2023demonstrations} do not require MDP access but still require access to  offline datasets possibly collected by non-expert policies, either with broad expert coverage or a large transition buffer.
Our work assumes that interacting with the environment does not incur costs; we leave a detailed analysis that takes into account environment interaction cost as future work.

\paragraph{Imitation Learning with interactive demonstrations}, first proposed by~\cite{ross2010efficient}, allows the expert to provide corrective feedback to the learner's action retroactively. Assuming low costs of expert recovery from mistakes ($\mu$-recoverability; recall Section~\ref{sec:prelims}), DAgger~\cite{ross2011reduction}, and following works~\cite{kim2013learning, ross2014reinforcement,sun2017deeply,cheng2018convergence, cheng2019predictor,rajaraman2021value} outperform traditional BC both theoretically and empirically. 
However, this efficiency demands substantial annotation effort~\cite{mandlekar2020human}.
Although DAgger~\cite{ross2011reduction} and some subsequent works~\cite{sun2017deeply,rajaraman2021value,swamy2022minimax,foster2024behavior} popularized the practice of annotating full trajectories, there has also been growing interest in state-wise annotations~\cite{nguyen2020active,li2022efficient,sekhari2023selective,li2023agnostic}, which appeared as early as~\cite{ross2010efficient,judah2014active}. In fact, practical applications of DAgger often adopt
partial trajectory annotation
in expert-in-the-loop~\cite{mandlekar2020human,spencer2022expert,liu2022robot} designs, as seen in~\cite{zhang2016query,kelly2019hg,hoque2021thriftydagger,torok2025greedy}, where issues such as inconsistencies caused by retroactive relabeling~\cite{laskey2017comparing} can be mitigated.
These methods often leverage human- or machine-gated expert interventions to ensure safety during data collection~\cite{zhang2017query,menda2017dropoutdagger}, provide more targeted feedback~\cite{menda2019ensembledagger,cui2019uncertainty}, and enable  learning on the fly~\cite{shi2024yell}. The use of selective state-wise queries aligns with our goal of promoting interactive imitation learning with efficient supervision and provable sample efficiency. 
We regard our contribution as providing a starting point for understanding this increasingly popular paradigm of
partial trajectory annotations.

\paragraph{Utilizing Offline Data for Interactive Learning.} Many practical deployments of interactive learning systems 
do not start from tabula rasa; instead, prior knowledge of various forms is oftentimes available. For example, combining offline data and interactive feedback has recently gained much popularity in applications such as training large language models to follow human instructions~\cite{dongrlhf,ouyang2022training}, and bandit machine translation~\cite{nguyen2017reinforcement}.
Many recent theoretical works in reinforcement learning try to quantify the computational and statistical benefit of combining offline and online feedback: for example,~\cite{li2023reward,tanhybrid} show provable reduction of sample complexity using hybrid reinforcement learning, using novel notions of partial coverage;~\cite{songhybrid} shows that under some structural assumptions on the MDP, hybrid RL can bypass computational barriers in online RL~\cite{kane2022computational}. Many works also quantify the benefit of utilizing additional offline data sources in the contextual bandit domain; for example,~\cite{oetomo2023cutting,sharma2020warm,zhang2019warm} study warm-starting contextual bandit learning using offline bandit logged data or supervised learning examples. While some variants of DAgger~\cite{zhang2017query,hoque2021thriftydagger} also operate in a hybrid setting, our work focuses on a fundamental formulation that explicitly accounts for the cost asymmetry between offline and interactive annotations~\cite{settles2008active}.



%% file: appendix.tex
\appendix

\section{Additional Related Work}

\input{additional_literature}

\section{Proof for \StateDagger}
\label{sec:dagger_state_annotation}

We first present two useful distance measures for pair of policies.

\begin{definition}[Trajectory-wise $L_1$-divergence]
    For a pair of Markovian policies $\pi$ and $\pi'$, define their trajectory-wise $L_1$-divergence as
\[
\lambda( \pi\parallel \pi')
:=
\EE^{\pi}  \mathbb{E}_{a'_1 \sim \pi'(\cdot|s_1), \ldots, a'_H \sim \pi'(\cdot|s_H)}  \sbr{
\sum_{h=1}^H \mathbb{I}(a_h \neq a'_h)
}.
\]
\label{def:traj_l1}
\end{definition}

$\lambda( \pi\parallel \pi')$ is the expected total number of actions taken by $\pi'$ that deviates from actions in trajectories induced by $\pi$. Note that $\lambda(\cdot||\cdot)$ is asymmetric, while the same concept 
is applied in offline and interactive IL~\cite{ross2010efficient,ross2011reduction} with different guarantees for $\lambda( \pi\parallel \pie)$ and $\lambda( \pie \parallel \pi)$ (see Lemma~\ref{lem: performance difference lemma base}).

\begin{definition}[State-wise Hellinger distance]
For a pair of policies $\pi$ and $\pi'$, define their state-wise Hellinger distance as
$ \EE_{s \sim d^{\pi}} \sbr{D^2_\text{H}(\pi(\cdot \mid s),\pi'(\cdot \mid s)) } $.
\label{def:state_h_dist}
\end{definition}
State-wise Hellinger distance represents
the expected Hellinger distance between the action distribution of $\pi$ and $\pi'$ on $s \sim d^\pi$. One notable feature 
here is that the distance is evaluated between $\pi(\cdot \mid s)$ and $\pi'(\cdot \mid s)$, unrelated to the original action $a$ taken by by $\pi$ when visiting $s$.
By 
Lemma~\ref{lem:hellinger-pointmass},
state-wise Hellinger distance with the expert policy is a constant factor equivalent to trajectory-wise $L_1$-divergence in the deterministic realizable expert setting.

In the following lemma, we show that the performance difference between the policy $\hat{\pi}$ returned by \StateDagger (Algorithm~\ref{alg:dagger_one_sample}) and the expert policy $\pie$ can be bounded by the online state-wise Hellinger estimation error:
\[
\Onstate_N := \sum_{n=1}^{N} \EE_{s \sim d^{\pi^n}} \sbr{D^2_\text{H}(\pi^n(\cdot \mid s),\pie(\cdot \mid s)) },
\]
where $\pi^n(\cdot \mid s)$ and $\pi^E(\cdot \mid s)$ denote the action distributions produced by the policies $\pi^n$ and $\pi^E$ at state $s$.



We are ready to prove the state-wise annotation complexity of Algorithm~\ref{alg:dagger_one_sample}: 

\begin{lemma}
For any MDP $\Mcal$, deterministic expert $\pie$, and sequence of policies $\cbr{\pi^n}_{n=1}^N$, then $\hat{\pi}$, the first-step uniform mixture of $\{\pi^1, \dots, \pi^N\}$ satisfies:
\[
J(\pie) - J(\hat{\pi}) \leq \mu H \cdot \frac{\Onstate_N}{N} .
\]
\label{lem:dagger_regret_new_state_wise_appendix}
\end{lemma}

\begin{proof}

By Lemma~\ref{lem: performance difference lemma base}, under the assumption of recoverability, the performance difference between $\hat{\pi}$ and the expert is bounded by
 \[
J(\pie) - J(\hat{\pi})
\leq 
\mu \cdot \lambda( \hat{\pi} \parallel \pie ),
 \]
where we recall the notation that
\[
\lambda( \pi \parallel \pie )
= 
\EE^{\pi} \sbr{
\sum_{h=1}^H \mathbb{I}(a_h \neq \pie(s_h))
}
=
\frac12 \sum_{h=1}^H \EE^{\pi} \| \pi(\cdot \mid s_h) - \pie(\cdot \mid s_h) \|_1.
\]



The proof follows by upper-bounding $\sum_{n=1}^N\lambda( \pi^n \parallel \pie )$ by $H\cdot \Onstate_N$. 
To this end, it suffices to show that for any stationary policy $\pi$,
\[ 
H \cdot \EE_{s \sim d^{\pi}}\sbr{ D_H^2( \pi(\cdot \mid s), \pie(\cdot \mid s) ) }
\geq
\frac12 \sum_{h=1}^H \EE^{\pi} \| \pi(\cdot \mid s) - \pie(\cdot \mid s) \|_1.
\]


Observe that $H \cdot \EE_{s \sim d^{\pi}}\sbr{ D_H^2( \pi(\cdot \mid s), \pie(\cdot \mid s) ) } = \sum_{h=1}^H \EE^{\pi} \sbr{ D_H^2( \pi(\cdot \mid s_h), \pie(\cdot \mid s_h) ) }$, we conclude the proof by applying Lemma~\ref{lem:hellinger-pointmass} with $p = \pi(\cdot \mid s_h)$ and $q = \pie(\cdot \mid s_h)$, which gives
\[
D^2_\text{H}(\pi(\cdot \mid s_h),\pie(\cdot \mid s_h))
\geq
\frac{1}{2}\|\pi(\cdot \mid s_h)-\pie(\cdot \mid s_h)\|_1.
\]

\end{proof}

\begin{theorem}[Theorem~\ref{thm:log_loss_dagger_main} Restated]
If \StateDagger (Algorithm~\ref{alg:dagger_one_sample}) is run with a state-wise expert annotation oracle $\Ostate$, an MDP $\Mcal$ where $(\Mcal, \pie)$ is $\mu$-recoverable, a policy class $\Bcal$ such that deterministic realizability (Assumption~\ref{assum:realizability}) holds, and the online learning oracle $\alg$ set as the exponential weight algorithm, then it returns $\hat{\pi}$ such that, with probability at least $1-\delta$,
\[
 \Onstate_{\Non}
 \leq 
\log(B) + 2\log(1/\delta),
\]
and furthermore,
the returned $\hat{\pi}$ satisfies
\[
J(\hat{\pi}) - J(\pie)
\leq 
\mu H\frac{ \log(B) + 2\log(1/\delta)}{\Non}.
\]

\label{thm:log_loss_dagger_appendix}
\end{theorem}

\begin{proof} 
Recall the each-step mixing in Definition~\ref{def:mixed_class_each_step}: since $\bar{\pi}_u$ is a each-step mixing policy, we have that $\forall h \in [H], s \in \Scal$, 
$\bar{\pi}_{u}(a|s) = \sum_{\pi \in \Bcal} u(\pi) \pi(a|s)$.


The loss function at each round that passed through online learning oracle $\alg$, evaluated at $\bar{\pi}_u \in \bar{\Pi}_\Bcal$, $\ell^n(\bar{\pi}_u)$,
 is of the form 
\[
\ell^n(\bar{\pi}_{u}) 
=
\log\rbr{\frac 1 {\bar{\pi}_{u}(a^{n,*}|s^n)} } 
=
\log\rbr{ \frac 1 {\sum_{\pi \in \Bcal}u(\pi)\pi(a^{n,*}|s^n) }},
\]
which is 1-exp-concave with respect to $u$. 
Thus, implementing $\alg$ using the exponential weights algorithm (Proposition~\ref{prop:log_loss_base}) achieves:
\[
\sum_{n=1}^{\Non} \log(1/\pi^n(a^{*,n} \mid s^n)) \leq \sum_{n=1}^{\Non} \log(1/\pie(a^{*,n} \mid s^n)) + \log(B) = \log(B). 
\]

Then, Lemma~\ref{lem:a14}, a standard online-to-batch conversion argument with $x^n = (s^n,h^n)$, $y^n = a^{*,n}$, $g_* = \pie$, and $\Hcal^n = \{o^{n'}\}_{n'=1}^n$, where $o^n = \left( s^n, a^n, a^{*,n} \right)$, implies that with probability at least $1-\delta$,
\[
\Onstate_{\Non}
=
\sum_{n=1}^{{\Non}} \EE_{s^n \sim d^{\pi^n}} \sbr{D^2_\text{H}(\pi^n(\cdot \mid s^n),\pie(\cdot \mid s^n)) } 
\leq
\log(B) + 2\log(1/\delta).
\]

The second part of the theorem follows by applying Lemma~\ref{lem:dagger_regret_new_state_wise_appendix}.
\end{proof}
\section{Proof for \WarmStartDagger}
\label{sec:hyil_appendix}

In this section, we analyze the guarantees of \WarmStartDagger under the realizable and deterministic expert assumption. We show that all intermediate policies, as well as the final returned mixture policy, induces distributions over trajectories that enjoy small Hellinger distance to the expert’s trajectory distribution, due to their agreement on the offline dataset. 
Our analysis builds on generalization guarantees of the maximum likelihood estimator (MLE).

\subsection{Preliminaries: First-step mixing and causally conditioned probabilities}
\label{sec:first-step-mixing-causal}

Our main analysis leverages the fact that the each-step mixing policies $\pi^n$'s  maintained by \WarmStartDagger can be viewed as first-step mixtures of policies in $\tilde{\Bcal}$ (recall Definition~\ref{def:policy_completion}) -- this subsection is dedicated to prove this result (Lemma~\ref{lem:step_mixture_to_first_step}).
We first recall the definition of first-step mixing of policies, adapted to our notations:


\begin{definition}[First-step mixing, e.g.~\cite{syed2007game}]
Given a distribution $u$ over a set of (possibly nonstationary) policies $\Pi$, denote $\pi_u$ as the first-step mixing policy induced by $u$: when rolling out $\pi_u$, first sample policy $\pi \sim u$, and follow $\pi$ for the entire episode. \footnote{Note that we use a notation $\pi_u$ different from every-step mixing policy notation $\bar{\pi}_u$.}
\end{definition}


Importantly, even when the policies in $\Pi$ are all Markovian,
$\pi_u$ may no longer be a Markovian policy: for example, the random draw of action $a_2$ does not only depend on the Markovian state $s_2$ but also the random policy $\pi$ drawn, which in turn may correlate with $a_1$. 

To simplify the notations in our development below, following~\cite{ziebart2010modeling}, we define causally conditional probabilities, when the agent uses Markovian policies or their first-step mixing: 




\begin{definition}
\label{def:causal-state-prob}
Given an MDP $\Mcal$, the causally conditional probability of state sequence $s_{1:H}$ given action sequence $a_{1:H-1}$, is defined as:
\[
\PP^\Mcal( s_{1:H} \parallel a_{1:H-1} )
= 
\rho(s_1) \prod_{h=1}^{H-1} P(s_{h+1} \mid s_h, a_h) . 
\]
Given a (first-step mixture) of Markovian policy, its causally conditional probability of state sequence $a_{1:H}$ given action sequence $s_{1:H}$ is defined as:
\begin{itemize}
\item For Markovian policy $\pi = \pi_{1:H}$, $\pi(a_{1:H} \parallel s_{1:H}) := \prod_{h=1}^H \pi_h(a_h | s_h)$.


\item For first-step mixing of Markovian policies $\pi_u$, 
$\pi_u(\cdot \parallel s_{1:H}) := \sum_{\pi \in \Bcal} u(\pi) \pi(\cdot \parallel s_{1:H})$.
\end{itemize}



\end{definition}


Note that $\pi(\cdot \| s_{1:H})$ 
and 
$\pi_u(\cdot \| s_{1:H})$
are valid probability distributions (e.g., $\sum_{a_{1:H}} \pi(a_{1:H} \| s_{1:H}) = 1$), however,  
the use of `$\parallel$' highlights its distinction from standard conditioning.
For example, when executing Markovian policy $\pi$, the conditional probability of the actions given the states $\PP^\pi(a_{1:H} | s_{1:H}) \propto \prod_{h=1}^H \pi_h(a_h | s_h) \prod_{h=1}^{H-1} P(s_{h+1} \mid s_h, a_h)$, which is clearly different from its causally conditional counterpart.

We have the following (perhaps folklore) lemma for causally conditional probability (e.g.~\cite{ziebart2010modeling}).

\begin{lemma}
\label{lem:prob-fact-causal}
For any Markovian policy $\pi$, 
\begin{equation}
\PP^\pi( s_{1:H}, a_{1:H} )
= 
\PP^\Mcal( s_{1:H} \parallel a_{1:H-1} ) \cdot \pi( a_{1:H} \parallel s_{1:H} ),
\label{eqn:causal-decomp}
\end{equation}
and for any first-step mixing of Markovian policies $\pi_u$, 
\begin{equation}
\label{eqn:causal-decomp-1st}
\PP^{\pi_u}( s_{1:H}, a_{1:H} )
= 
\PP^\Mcal( s_{1:H} \parallel a_{1:H-1} ) \cdot {\pi_u}( a_{1:H} \parallel s_{1:H} ).
\end{equation}
\end{lemma}
\begin{proof}
Eq.~\eqref{eqn:causal-decomp} follows by noting that both sides are equal to 
\[ 
\rho(s_1) \prod_{h=1}^{H-1} P(s_{h+1} \mid s_h, a_h) \prod_{h=1}^H \pi_h(a_h \mid s_h).
\]
Eq.~\eqref{eqn:causal-decomp-1st} follows by noting that both sides are equal to 
\[
\sum_{\nu} 
u(\nu) \rho(s_1) \prod_{h=1}^{H-1} P(s_{h+1} \mid s_h, a_h) \prod_{h=1}^H \nu_h(a_h \mid s_h).
\qedhere
\]
\end{proof}




\begin{lemma}
\label{lem:step_mixture_to_first_step}
%
If $\bar{\pi}_u$ is an each-step policy in $\bar{\Pi}_\Bcal$, then there exists some first-step mixture of policies in $\tilde{\Bcal}$ that is equivalent to $\bar{\pi}_u$, i.e., they induce the same  distribution over all length-$H$ trajectories.
\end{lemma}


\begin{proof}
Define $\mu$ 
as a distribution over class 
\(\tilde{\Bcal}\) with 
\(\mu (\nu) := \prod_{h=1}^H u(\nu_h)\), for every $\nu = \nu_{1:H}$ in $\tilde{\Bcal}$. Consider the joint action distribution under \(\bar{\pi}_u\), which samples \(\pi_h \sim u\) independently for each step and executes \(a_h \sim \pi_h(\cdot \mid s_h)\). The resulting causally conditional distribution 
over actions given state sequence $s_{1:H}$ is
\[
\bar{\pi}_u(a_{1:H} \| s_{1:H}) = \prod_{h=1}^H \left( \sum_{\pi_h \in \Bcal} u(\pi_h)\, \pi_h(a_h \mid s_h) \right).
\]

On the other hand, under the first-step mixture policy \(\pi_{\mu}\) over \(\tilde{\Bcal}\), a full tuple \(\nu = (\nu_1, \dots, \nu_H)\) is sampled once from $\mu$, and actions are drawn as \(a_h \sim \nu_h(\cdot \mid s_h)\). The resulting action distribution is
\[
\pi_{\mu}(a_{1:H} \| s_{1:H}) = \sum_{\nu \in \tilde{\Bcal}} \mu (\nu) \prod_{h=1}^H \nu_h(a_h \mid s_h).
\]

Expanding the sum yields
\[
\sum_{(\nu_1,\dots,\nu_H) \in \Bcal^H} \left( \prod_{h=1}^H u(\nu_h) \nu_h(a_h \mid s_h) \right) = \prod_{h=1}^H \sum_{\pi_h \in \Bcal} u(\nu_h)\, \nu_h(a_h \mid s_h),
\]
by the distributive property and independence of the product.

Therefore, $\bar{\pi}_u(a_{1:H} \| s_{1:H}) = \pi_{\mu}(a_{1:H} \| s_{1:H}),
$
and both policies induce the same trajectory distribution by Lemma~\ref{lem:prob-fact-causal}. 
\end{proof}

\subsection{Proof of Theorem~\ref{thm:wstagger_main}}


\begin{lemma}
\label{lem:j_stability}
Let $\tilde{\Bcal}_{\mathrm{bc}} := \{ \pi = \pi_{1:H} \in \tilde{\Bcal}: \pi_h(s_h) = a_h,\ \forall h \in [H], \forall (s, a)_{1:H} \in D_{\mathrm{off}} \}$ be the set of policies in $\tilde{\Bcal}$ that agree with the expert on the offline dataset of $\Noff$ iid expert trajectories. Assume the expert policy $\pi^E$ is deterministic and realizable. Then, with probability at least $1 - \delta$:
for all $\pi$ in $\tilde{\Bcal}_{\mathrm{bc}}$, 
\begin{equation}
J(\pie) - J(\pi) \leq O\left( \frac{ R \log(\tilde{B}/\delta)}{\Noff} \right).
\label{eqn:tilde-bbc}
\end{equation}
Consequently, for all $\pi^n$'s computed in \WarmStartDagger (Algorithm~\ref{alg:warm_il}), it holds that:
\begin{equation}
J(\pie) - J(\pi^n) \leq O\left(  \frac{ R \log(\tilde{B}/\delta)}{\Noff} \right).
\label{eqn:warmstagger-policies}
\end{equation}


\end{lemma}

\begin{proof}

Eq.~\eqref{eqn:tilde-bbc} is a direct consequence of behavior cloning's guarantee applied to Markovian policy class $\tilde{\Bcal}$
(\cite[][Corollary 2.1]{foster2024behavior}).\footnote{Our presentation of Theorem~\ref{thm:bc-guarantee} uses a less general interpretation of that corollary, by restricting the policy class to be stationary.}

Eq.~\eqref{eqn:warmstagger-policies} follows from Lemma~\ref{lem:step_mixture_to_first_step} that $\PP^{\pi^i}$ is a convex combination of $\PP^{\pi}$'s for $\pi$'s in $\tilde{\Bcal}_{\mathrm{bc}}$, as well as the fact that the expected return function $J(\pi)$ is linear in the trajectory distribution $\PP^\pi$.
\end{proof}

\begin{theorem}[Theorem~\ref{thm:wstagger_main} restated]

If Algorithm~\ref{alg:warm_il} is run with a deterministic expert policy $\pi^E$, an MDP $\Mcal$ such that $(\Mcal, \pi^E)$ is $\mu$-recoverable, a policy class $\Bcal$ such that deterministic realizability holds, and the online learning oracle $\alg$ set as the exponential weight algorithm, then it returns $\hat{\pi}$ such that, with probability at least $1-\delta$,
\begin{equation}
J(\pie)- J(\hat{\pi}) \leq O\left( \min \left( \frac{R \log(\tilde{B}/\delta)}{\Noff}, \frac{\mu H \log(B_{\mathrm{bc}}/\delta)}{ N_{on}} \right) \right),
\label{eqn:warm-stagger-main}
\end{equation}
\end{theorem}

\begin{proof}

    
    We bound $J(\pie) - J(\hat{\pi})$
    by the two terms in the minimum expression on the right hand side  of Eq.~\eqref{eqn:warm-stagger-main} respectively.

    For the first term, we recall from Lemma~\ref{lem:j_stability} that with probability $1-\delta/2$, for all policies $\pi^n$'s, $J(\pie) - J(\pi^n) \leq O\left(\frac{R \log(\tilde{B}/\delta)}{\Noff} \right)$. Since $J(\hat{\pi}) = \frac{1}{N} \sum_{n=1}^N J(\pi^n)$, 
    we have,
    \[
    J(\pie) - J(\hat{\pi}) \leq O\left(\frac{R \log(\tilde{B}/\delta)}{\Noff} \right).
    \]

    For the second term, we note that by definition $\pie \in \Bcal_{\mathrm{bc}}$. 
    Thus, by applying Theorem~\ref{thm:log_loss_dagger_main}, we conclude that with probability at least $1-\delta/2$, the returned $\hat{\pi}$ satisfies
    \[
    J(\hat{\pi}) - J(\pie) \leq O\left( \frac{\mu H \log(B_{\mathrm{bc}}/\delta)}{ \Non } \right).
    \]

    Together, we conclude our proof by applying a union bound.
\end{proof}

\section{Proof for Theorem~\ref{thm:tabular-ex}}
\label{sec:hyil_benefit_appendix}

We formally define the MDP $\Mcal$ and the expert policy from Section~\ref{sec:hyil_benefit_main}, where the expert policy $\pie$ is deterministic 

\begin{itemize}
    \item \textbf{State Space} $\Scal = \mathbf{E} \cup \mathbf{E'} \cup \mathbf{B} \cup \mathbf{B'}$, where:
    \begin{itemize}
        \item $\mathbf{E}$: ideal expert states, $|\mathbf{E}| = N_0$;
        \item $\mathbf{E'}$: recoverable expert states, $|\mathbf{E'}| = N_1$;
        \item $\mathbf{B} = \{\mathbf{b}\}$: absorbing failure state (unrecoverable);


        
        \item $\mathbf{B'}= \{\mathbf{b'}\}$: recoverable reset state.
    \end{itemize}

    \item \textbf{Action Space} $\Acal$: contains $A = |\Acal|$ discrete actions. For each state $s \in \Scal$, there is a unique action $\pie(s)$ taken by the expert.
    
    \item \textbf{Episode length $H$.}

    \item \textbf{Initial State Distribution \(\rho\)}:
    \[
    \rho(s) = \frac{1}{(1+\beta)N_0} \; \text{for all} \; s \in \bE, \quad \rho(s) = \frac{\beta}{(1+\beta)N_1}  \; \text{for all} \; s \in \bE'.
    \]

    \item \textbf{Transition Dynamics}:
    \begin{itemize}



        
        
        \item $s \in \mathbf{E}$:
        \begin{itemize}
            \item $a = \pie(s)$: with probability $1 - \beta$, transitions to a uniformly random $s' \in \mathbf{E}$;
            with probability $\beta$, transitions to
            a uniformly random
            $s' \in \mathbf{E'}$ ;
            \item $a \neq \pie(s)$: transitions to $\mathbf{b}$.
        \end{itemize}

        \item $s \in \mathbf{E'}$:
        \begin{itemize}
            \item $a = \pie(s)$: transitions to a uniformly random $s' \in \mathbf{E}$;
            \item $a \neq \pie(s)$: transitions to $\mathbf{b'}$.
        \end{itemize}

        \item $s \in \mathbf{B} = \cbr{ \mathbf{b} }$: absorbing for all actions. 
        Specifically, 
        $P(s' = \bob| s=\bob, a) = 1$.

        \item $s \in \mathbf{B'} = \cbr{ \mathbf{b}' }$:
        \begin{itemize}
            \item $a = \pie(\mathbf{b'})$: transitions to a uniformly random $s' \in \mathbf{E}$;
            \item $a \neq \pie(\mathbf{b'})$: remains in $\mathbf{b'}$.
        \end{itemize}
    \end{itemize}

    \item \textbf{Reward Function}:
    For theoretical analysis, we consider the following reward function $R_1$:
        \[
        R_1(s, a) =
        \begin{cases}
        1 & \text{if } s \in \mathbf{E} \cup \mathbf{E'} \\
        1 & \text{if } s = \mathbf{b'} \text{ and } a = \pie(s) \\
        0 & \text{otherwise}
        \end{cases}
        \]
        
    \item \textbf{Specification of Parameters}:
    In the following proofs, we let $H \geq \max(50, \frac{5}{4}\log(10N_0))$, $A \geq 10H$, $\beta = \frac{8}{H-8}$, 
    $N_1 \geq \max(500, 160 N_0)$.

        
\end{itemize}


We also make the following assumption on the policies produced in our learning algorithms (BC, \StateDagger, and \WarmStartDagger).
At any stage of learning, denote the set of states that are annotated by the expert by $\Scal_{\mathrm{annotated}}$. Given the annotated (state, expert action) pairs, the learner calls some offline or online learning oracle $\mathbb{A}$ to obtain a policy $\pi$. We require $\pi$'s behavior as follows:
\[
\pi(\cdot \mid s) = 
\begin{cases}
\pie(\cdot \mid s), & 
s \in \Scal_{\mathrm{annotated}}, \\
(\frac 1 A, \ldots, \frac 1 A), & s \notin \Scal_{\mathrm{annotated}}.
\end{cases}
\]
In other words, $\pi$ follows the expert's action whenever such information is available; otherwise, it takes an action uniformly at random.


Denote by $d_h^\pi(s) = \PP^\pi(s_h = s)$ policy $\pi$'s state visitation distribution in $\Mcal$ at step $h$.
We next make a simple observation that by the construction of $\Mcal$, the expert policy's state-visitation distribution in $\Mcal$ is stationary over
all steps: 
\begin{observation}
For MDP $\Mcal$, $d_h^{\pi^E}$ equals $\rho$, for all $h \in [H]$. 
\label{obs:visitation_distribution}
\end{observation}
\begin{proof}
Recall the initial state distribution:
\[
\rho(s) =
\begin{cases}
\frac{1}{N_0(1+\beta)},\; & s\in\mathbf{E}, \\
 \frac{\beta}{N_1(1+\beta)},\; & s\in\mathbf{E'}, \\
 0, & \text{otherwise}.
\end{cases}
\]
Under $\pie$, states $\mathbf{b}$ and $\mathbf{b}'$ are never reached. And the induced transition kernel on $\mathbf{E}\cup\mathbf{E'}$ satisfies:
\[
P(s'|s,\pi^E(s)) =
\begin{cases}
\frac{1-\beta}{N_0}, & s\in\mathbf{E},\; s'\in\mathbf{E},\\
\frac{\beta}{N_1}, & s\in\mathbf{E},\; s'\in\mathbf{E'},\\
\frac{1}{N_0}, & s\in\mathbf{E'},\; s'\in\mathbf{E},\\
0, & \text{otherwise}.
\end{cases}
\]

For any fixed $s \in \mathbf{E}$, using the kernel above,
\[
\sum_{s'} \rho(s')P(s\mid s',\pi^E(s'))
= \sum_{s'\in\mathbf{E}} \frac{1}{N_0(1+\beta)}\cdot \frac{1-\beta}{N_0}
  + \sum_{s'\in\mathbf{E'}} \frac{\beta}{N_1(1+\beta)}\cdot \frac{1}{N_0}
= \frac{1}{N_0(1+\beta)}
= \rho(s).
\]
Similarly, for any fixed $s \in \mathbf{E'}$,
\[
\sum_{s'} \rho(s')P(s\mid s',\pi^E(s'))
= \sum_{s'\in\mathbf{E}} \frac{1}{N_0(1+\beta)}\cdot \frac{\beta}{N_1}
= \frac{\beta}{N_1(1+\beta)}
= \rho(s).
\]
Thus, for any $s \in \mathbf{E} \cup \mathbf{E'}$,
\[
\rho(s) = \sum_{s'} \rho(s')P(s\mid s',\pi^E(s')).
\]
Hence $\rho$ is a stationary distribution for the Markov chain induced by rolloing out $\pi^E$ in $\Mcal$.
Since $d^{\pi^E}_1 = \rho$ and the dynamics are time-homogeneous, induction gives that $\forall h\in[H]$, $d^{\pi^E}_h = \rho$.
%

\end{proof}

The following is the main theorem of this section; setting $N_0 = \Theta(N_1)$ (in which case both $N_0$ and $N_1$ are $\Theta(S)$) gives Theorem~\ref{thm:tabular-ex} in our main paper.

\begin{theorem}[Strengthening of Theorem~\ref{thm:tabular-ex}]
To achieve smaller than $\frac{H}{2}$ suboptimality 
compared to expert in MDP $\Mcal$ with probability $\frac12$:
\begin{itemize}
\item Behavior Cloning (BC) using offline expert trajectories requires
    \[
    \Noff = \Omega(N_1)
    \quad \text{with total annotation cost } \Omega(H N_1).
    \]

    \item \StateDagger that collects interactive state-wise annotations requires
    \[
    \Non = \Omega(H N_0)
    \quad \text{with total annotation cost } \Omega(C H N_0).
    \]

\end{itemize}

In contrast, \WarmStartDagger 
learns a policy that achieves expert performance with probability at least $\frac12$, using
\begin{equation}
\begin{aligned}
    \Noff &= O(\frac{N_0}{H} \log(N_0)) \quad \text{expert trajectories, and} \\
    \Non  & \leq 3 \quad \text{interactive annotations,} \\
    &\text{with total annotation cost } \tilde{O}(N_0 +C).
\end{aligned}
\end{equation}
\end{theorem}




\begin{proof}
The proof is divided into three parts: 

First, by Lemma~\ref{lem:BC-lower}, we show that in $\Mcal$,
Behavior Cloning requires $\Omega(H N_1)$ expert trajectories to achieve suboptimality \( H/2 \) with probability $\frac12$.

Next, in Lemma~\ref{lem:stagger_lowerbound}, we show that \StateDagger, which rolls out the learner policy and queries the expert on only one state sampled uniformly from its learned policy's rollout, requires \( \Non = \Omega(H N_0) \) interactive annotations to achieve suboptimality no greater than \( H/2 \) with probability $\frac12$. 

 Finally, by Lemma~\ref{lem:warmdagger_success_r1}, we demonstrate that \WarmStartDagger achieves expert performance using \( O(\frac{N_0}{H} \log(N_0)) \) offline expert tarjectories and 3 interactive annotations with probability $\frac12$.
\end{proof}

\subsection{Lower Bound for Behavior Cloning}

Throughout this subsection, we denote by $\mathbf{E}_{\mathrm{annotated}}'$ the set of states in $\mathbf{E'}$ that are visited and annotated by the expert's $\Noff$ offline trajectories.

\begin{lemma}[BC suboptimality lower bound]\label{lem:BC-lower}
Consider the MDP $\Mcal$ and the expert policy $\pie$ constructed as above. If Behavior Cloning uses
\[
\Noff < \frac{N_1}{160} 
\]
iid expert trajectories,
then with probability at least $\frac12$, the suboptimality of its returned policy $\hat{\pi}$ is lower bounded by:
\[
 J(\pie) - J(\hat{\pi}) \geq \frac{H}{2}.
\]
\end{lemma}

\begin{proof}

First, given policy $\hat{\pi}$, we define a modified policy
$\tilde{\pi}$ that agrees with $\hat{\pi}$ everywhere except that it always take the expert's action on $\mathbf{E}$. Note that if $\hat{\pi}$ ever takes a wrong action at a state in $\mathbf{E}$,
the trajectory deterministically falls into the absorbing bad state $\mathbf{b}$,
yielding even smaller return than $\tilde{\pi}$. Thus, $J(\tilde{\pi}) \ge J(\hat{\pi})$, so it suffices to prove the claimed lower
bound for $\tilde{\pi}$, which we show for the remainder of this proof.



In the following, we show that an insufficient number of expert trajectories leads to small $|\mathbf{E}_{\mathrm{annotated}}'|$ i.e., poor coverage on $\mathbf{E'}$. This, in turn, causes policy $\tilde{\pi}$ to frequently fail to recover and get trapped in the absorbing bad state 
$\mathbf{b}'$, incurring a large suboptimality compared to the expert.

By Lemma~\ref{lem:S1-limit}, when $\Noff$ is below the stated threshold, with probability $\frac12$,
\[
|\mathbf{E}_{\mathrm{annotated}}'| \leq \frac{1}{10} |\mathbf{E}'|
\]
We henceforth condition on this happening. In this case, recall that our offline learning oracle is such that $\tilde{\pi}$ takes actions uniformly at random in states in   
$\mathbf{E}' \setminus \mathbf{E}_{\mathrm{annotated}}'$. 
Suppose we roll out the 
policy $\tilde{\pi}$ in $\Mcal$; 
let $\tau$ be the first step such that
$s_\tau \in \mathbf{E'}$ (and $\tau := H+1$ if no such step exists).



By Lemma~\ref{lem:S1-visit},
\[
\Pr_{\Mcal, \tilde{\pi}}(\tau \le H/5) \ge 0.79 , \; \text{and} \; \Pr_{\Mcal, \tilde{\pi}}\big( s_\tau \notin \mathbf{E}_{\mathrm{annotated}}' \mid \tau \le H/5 \big) \ge 0.9.
\]

We henceforth condition on the event $\{\tau \le H/5,\; s_\tau \notin \mathbf{E}_{\mathrm{annotated}}'\}$ when rolling out $\tilde{\pi}$.
Applying Lemma~\ref{lem:unannotated-wrong-general}, we have, with
probability at least $0.9$,
 $\tilde{\pi}$ takes a wrong action at $s_\tau$ and transitions to $\mathbf{b'}$,
and subsequently never takes the expert recovery action at $\mathbf{b'}$. Therefore the
trajectory remains in $\mathbf{b'}$ from step $H/5$ onward, yielding zero
reward for at least $\frac{4H}{5}$ steps.

Multiplying all factors, with probability greater than $\frac12$,
\[
J(\pi^E) - J( \tilde{\pi} )
\ge 
\underbrace{0.79}_{\text{reach }\mathbf{E'}}
\times 
\underbrace{0.9}_{\text{unannotated } s_\tau}
\times 
\underbrace{0.9}_{\text{action errors}}
\times 
\underbrace{0.8H}_{\text{zero reward}}
> \frac{H}{2},
\]
which concludes the proof.
\end{proof}

\begin{lemma}[Bounded $\mathbf{E'}$ coverage]\label{lem:S1-limit}
If the number of expert trajectories satisfies:
\[
\Noff \leq \frac{N_1}{160}, 
\]
then,
with probability at least $\frac12$,
\[
|\mathbf{E}_{\mathrm{annotated}}'| \;\le\; \frac{N_1}{10}.
\]
\end{lemma}


\begin{proof}

Recall that by Observation~\ref{obs:visitation_distribution}, $d_h^{\pi^E}(s) = \frac{\beta}{N_1(1+\beta)} = \frac{8}{H N_1}$ for $s \in \mathbf{E'}$, where $\beta = \frac{8}{H-8}$. 
Denote $X := |\mathbf{E}_{\mathrm{annotated}}'|$ to be the number of annotated states in $\mathbf{E'}$.
Consider $\Noff$ expert trajectories, each of length $H$, and let $s_{i,h}$
be the state at step $h$ of trajectory $i$. Fix any $s \in \mathbf{E'}$. By a union bound over all steps and available offline expert trajectories,
\[
\Pr(s \in \mathbf{E}_{\mathrm{annotated}}') 
= \Pr\Big( \bigcup_{i=1}^{\Noff} \bigcup_{h=1}^H \{ s_{i,h} = s \} \Big)
\le \sum_{i=1}^{\Noff} \sum_{h=1}^H \Pr(s_{i,h} = s)
= \sum_{i=1}^{\Noff} \sum_{h=1}^H d_h^{\pi^E}(s)
= \frac{8 \Noff}{N_1}.
\]
Under the assumption $\Noff \le N_1/160$, this gives
\[
\Pr(s \in \mathbf{E'}_{\mathrm{annotated}})  \le \frac{1}{20}.
\]

Let $Z_s$ be the indicator that state $s$ appears in the expert trajectories, by linearity of expectation,  
\[
\mathbb{E}[X] = \sum_{s \in \mathbf{E'}} \mathbb{E}[Z_s]
= \sum_{s \in \mathbf{E'}} \Pr(s \in \mathbf{E'}_{\mathrm{annotated}})
\le \frac{N_1}{20}.
\]

Applying Markov’s inequality at threshold $N_1/10$,
\[
\Pr\Big( X > \frac{N_1}{10} \Big)
\le \frac{\mathbb{E}[X]}{N_1/10}
\le \frac{N_1/20}{N_1/10}
= \frac{1}{2}.
\]
Equivalently,
\[
\Pr\Big( X \le \frac{N_1}{10} \Big) \ge \frac{1}{2}.
\qedhere
\]
\end{proof}

\begin{lemma}[First $\mathbf{E'}$ visit]\label{lem:S1-visit}
For 
the MDP $\Mcal$, and any policy $\pi$ that agrees with $\pie$ on $\mathbf{E}$:
\[
\Pr_{\Mcal, \pi} \big( \exists h \in [H/5], s_t \in \mathbf{E}' \;\big) \geq 0.79
\]

\end{lemma}

\begin{proof}
Since $\pi$ agrees with $\pie$ on $\mathbf{E}$, 
\[
\Pr_{\Mcal, \pi} \big( \exists h \in [H/5], s_t \in \mathbf{E}' \;\big) 
=
\Pr_{\Mcal, \pie} \big( \exists h \in [H/5], s_t \in \mathbf{E}' \;\big). 
\]
It suffices to consider the expert policy $\pie$'s visitation.
By Observation~\ref{obs:visitation_distribution}, the state visitation distribution for $\pie$ satisfies that for all $h \in [H]$, 
\[
d^{\pi^E}_h(s) =
\begin{cases}
\frac{1}{N_0(1+\beta)},\; & s'\in\mathbf{E}, \\
 \frac{\beta}{N_1(1+\beta)},\; & s'\in\mathbf{E}, \\
 0, & \text{otherwise}.
\end{cases}
\]
The probability that \textit{no} $s \in \mathbf{E'}$ is visited in $\frac{H}{5}$ steps is:




\begin{equation}
\begin{aligned}
\frac 1 {1+\beta} \cdot (1-\beta)^{\frac H 5 - 1}   
= &
\frac 1 {(1+\beta)(1-\beta)}\left(\frac{1}{1+\beta}\right)^{\frac{H}{5}} \\
= &
\frac 1 {1-\beta^2} \left(1 - \frac{8}{H}\right)^{\frac{H}{5}} \\
\leq & 
\frac 1 {1-\frac{64}{(H-8)^2}} e^{-\frac{8}{H} \cdot \frac{H}{5}} \\
<  & 
1.038 \cdot e^{-\frac{8}{5}} < 0.21,
\end{aligned}
\end{equation}

where we apply $\beta = \frac{8}{H-8}$, $1 - x \leq e^{-x}$, and $H \geq 50$. Thus:
\[
\Pr_{\Mcal, \pie} \big( \exists h \in [H/5], s_t \in \mathbf{E}' \;\big) 
> 1 - 0.2095 = 0.79
\qedhere
\]
\end{proof}


\begin{lemma}[Action Errors on Unannotated States]\label{lem:unannotated-wrong-general}

Consider MDP $\Mcal$ with action space
size $A \geq 10H$.
\[
\Pr_{\Mcal, \tilde{\pi}}\big( \forall h \in [H]: 
s_h \in (\mathbf{E}' \setminus \mathbf{E}_{\mathrm{annotated}}') \cup \mathbf{B'} \implies
\ a_h \neq \pi^E(s_h) \big)
\;\ge\; 0.9.
\]
\end{lemma}

\begin{proof}


Since for any state in $(\mathbf{E}' \setminus \mathbf{E'}_{\mathrm{annotated}}) \cup \mathbf{B'}$, 
policy $\tilde{\pi}$ selects the expert action with probability exactly $\frac{1}{A}$,
thus, 
\begin{align*}
& \Pr_{\Mcal, \tilde{\pi}}\big(\exists h\in [H] :\
s_h \in (\mathbf{E}' \setminus \mathbf{E}_{\mathrm{annotated}}') \cup \mathbf{B'} \text{ and }
a_h =\pi^E(s_h)\big) \\
\leq & 
\sum_{h=1}^H 
\Pr_{\Mcal, \tilde{\pi}}\big(
s_h \in (\mathbf{E}' \setminus \mathbf{E}_{\mathrm{annotated}}') \cup \mathbf{B'}
\text{ and }
a_h =\pi^E(s_h)\big) \\
\leq & \sum_{h=1}^H 
\Pr_{\Mcal, \tilde{\pi}}\big(
a_h =\pi^E(s_h)
\mid 
s_h \in (\mathbf{E}' \setminus \mathbf{E}_{\mathrm{annotated}}') \cup \mathbf{B'}
\big)
\leq \frac{H}{A} \leq 0.1,
\end{align*}
where we recall that $A \geq 10H$ in our construction. The lemma now follows by taking the complement of the probability above.
\end{proof}




\subsection{Lower Bound for \StateDagger}

Throughout the proof, denote by $\mathbf{E}_{\mathrm{annotated}}$ the final set of states in $\mathbf{E}$ on which \StateDagger have requested expert annotations.

\begin{lemma}[\StateDagger suboptimality lower bound]
\label{lem:stagger_lowerbound}
Consider the MDP $\Mcal$ from Section~\ref{sec:hyil_benefit_main} with $H \geq 50$, $A \geq 10H$, and $\beta = \frac{8}{H-8}$. If \StateDagger collects no more than
\[
\Non \leq \frac{H N_0}{12}
\]
interactive state-wise annotations, then, with probability at least $\frac12$, the returned policy $\hat{\pi}$ suffers suboptimality at least
\[
J(\pie) - J(\hat{\pi}) \geq \frac{H}{2}.
\]
\end{lemma}

\begin{proof}
By Lemma~\ref{lem:E-bound-stagger}, if \StateDagger collects at most $\frac{H N_0}{12}$ interactive state-wise annotations as above, then with probability at least $\frac 12$, 
$|\mathbf{E}_{\mathrm{annotated}}|$,
the number of distinct states in $\mathbf{E}$ annotated is fewer than $N_0/3$.
Consider a random rollout of $\hat{\pi}$, and define the following events:
\begin{align*}
F_1 &:= \cbr{ s_1 \notin \mathbf{E}_{\mathrm{annotated}} }, 
\\
F_2 &:= \{ a_1 \neq \pie(s_1)\}.
\end{align*}


We now lower bound their probabilities:
\begin{align*}
\Pr_{\hat{\pi}}(F_1) &\ge \frac{2}{3(1+\beta)},\\
\Pr_{\hat{\pi}}(F_2 \mid F_1) &\ge 1 - \frac{1}{A} \geq 1- \frac{1}{10H}.
\end{align*}








Conditioned on the two events, the 
agent will get trapped at state $\mathbf{b}$ from step 2 on, and thus its
conditional expected return satisfies
\[
\EE_{\hat{\pi}}\sbr{\sum_{h=1}^H r_h \mid F_1, F_2} \leq 1.
\]
Also, by the definition of the reward function $R_1$, $J(\pi^E) = H$.
Thus, 
\begin{align*}
J(\pi^E) - J(\hat{\pi})
= &
\EE_{\hat{\pi}}\sbr{ H - \sum_{h=1}^H r_h } \\
\geq &
\PP_{\hat{\pi}}(F_1, F_2) \cdot \EE_{\hat{\pi}}\sbr{ H - \sum_{h=1}^H r_h \mid F_1, F_2 }  \\
\geq &
\frac{2}{3} \cdot \frac{H-8}{H}\cdot \frac{10H-1}{10H} \cdot (H-1) 
\geq \frac{H}{2},
\end{align*}
where we use our setting of $\beta = \frac{8}{H-8}$ and apply $H\geq 50$.
\end{proof}


\begin{lemma}[Bounded $\mathbf{E}$ coverage under \StateDagger]
\label{lem:E-bound-stagger}
Suppose \StateDagger collects at most
\[
\Non \leq \frac{H N_0}{12}
\]
interactive annotations. 
Then,
%
\[
\Pr\left(|\mathbf{E}_{\mathrm{annotated}}| \geq \frac{N_0}{3} \right) \leq \frac12.
\]
\end{lemma}

\begin{proof}



In this proof, we say that state $s$ is \emph{annotated at iteration $i$}, if it has been annotated by the expert before iteration $i$ (excluding iteration $i$). 
Since each iteration of \StateDagger samples one state uniformly from the current policy's rollout for annotation, we denote
 the indicator of 
 expert annotating
 an unannotated state from 
 \(\mathbf{E}\) at iteration $i$ as \(X_i \in \{0, 1\}\). With this notation, we have the total number of annotated states by the end of iteration $\Non$ as
\[
|\mathbf{E}_{\mathrm{annotated}}| = \sum_{i=1}^{\Non} X_i.
\]

Let \(\mathcal{F}_{j}\) be 
the sigma-algebra generated by all information seen by \StateDagger up to iteration $j$.
We now upper bound the expected value of \(X_i\) conditioned on $\Fcal_{i-1}$ for each $i$. 

Denote by $Y_i$ the number of unannotated states in $\mathbf{E}$ visited by round $i$'s rollout (by policy $\pi^i$).
We claim that conditioned on $\Fcal_{i-1}$, $Y_i$ is stochastically dominated by a geometric distribution with parameter $\frac{A-1}{A}$. Indeed, whenever an unannotated state in $\mathbf{E}$ is encountered when rolling out $\pi^i$, the probability that the agent takes a wrong action is $\frac{A-1}{A}$; if so, the agent gets absorbed to $\mathbf{b}$ immediately, and thus never sees any new unannotated states in this episode. In summary, 
\[
\EE[Y_i \mid \Fcal_{i-1}] \leq \EE_{Z \sim \mathrm{Geometric}(\frac{A-1}{A})}[Z] = \frac{A}{A-1} \leq 2.
\]

Since the state sampled for expert annotation is uniformly at random from the trajectory, conditioned on $Y_i$,
the probability that it lands on a an unannotated state in $\mathbf{E}$ is at most 
$\frac{Y_i}{H}$. Hence,
\[
\mathbb{E}[X_i \mid \mathcal{F}_{i-1}] 
=
\EE\sbr{ \mathbb{E}[X_i \mid Y_i, \mathcal{F}_{i-1}] \mid \Fcal_{i-1} }
= 
\EE\sbr{ \frac{Y_i}{H} \mid \Fcal_{i-1} } 
\leq
\frac{2}{H}.
\]
By linearity of expectation:
\[
\mathbb{E}[ |\mathbf{E}_{\mathrm{annotated}} | ] = \sum_{i=1}^{\Non} \mathbb{E}[X_i] \leq \frac{2 \Non}{H}.
\]

Applying Markov's inequality:
\[
\Pr(|\mathbf{E}_{\mathrm{annotated}}| \geq N_0 / 3) \leq \frac{\mathbb{E}[X]}{N_0 / 3} \leq \frac{6 \Non}{H N_0}
\leq \frac12,
\]
where the last inequality is by our assumption that $\Non \leq \frac{H N_0}{12}$.
This completes the proof.
\end{proof}


\subsection{Upper Bound for \WarmStartDagger}

\begin{lemma}[Hybrid IL achieves expert performance under $R_1$]
\label{lem:warmdagger_success_r1}
Consider the MDP $\Mcal$ and expert policy $\pie$ as above, then, with probability at least $1/2$, \WarmStartDagger outputs a policy that achieves expert performance using $\Noff = O\left( \frac{N_0}{H} \log(N_0) \right)$ offline expert trajectories and $\Non \leq 3$ interactive annotations.
\end{lemma}
\begin{proof}
We divide the proof into four steps. Throughout, we denote by $\mathbf{E}_{\mathrm{annotated}}$ and
$\mathbf{E}_{\mathrm{annotated}}'$ the subsets of $\mathbf{E}$ and
$\mathbf{E'}$, respectively, that are annotated by the $\Noff$ offline expert demonstration trajectories.

\paragraph{First}, we state a high probability event for the $\Noff$ offline expert demonstration trajectories to provide annotations on all states in $\mathbf{E}$. 
Define event
\[
F_3 := \{ \mathbf{E}_{\mathrm{annotated}} = \mathbf{E} \}.
\]
By Lemma~\ref{lem:E-expert-coverage-bound}, the choice
\[
\Noff = \frac{N_0}{(1-\beta)H} \log\!\Big(10 N_0\Big)
\]
ensures that $\Pr(F_3) \ge 0.9$.  On event $F_3$, the learner takes the correct
action on every state in $\mathbf{E}$.


\paragraph{Next}, we define the event under which only a small fraction of $\mathbf{E'}$
is covered by the expert.  
With our setting of MDP parameters, the number of offline trajectories satisfies
\[
\Noff = \frac{N_0}{(1-\beta)H} \log\!\Big(10 N_0\Big) \leq \frac{4}{5(1-\beta)}  N_0
=
\frac{4(H-8)}{5(H-16)}  N_0
\leq N_0 < \frac{N_1}{160},
\]
where in the first inequality we apply $H \ge \frac{5}{4}\log(10N_0)$; in the second inequality, we  apply $\frac{H-16}{H-8} < \frac{4}{5}$ when $H \geq 50$; and we apply $N_0 < \frac{N_1}{160}$ in the third inequality by the MDP setting. Thus, by Lemma~\ref{lem:S1-limit},
\[
\mathbb{E}\!\left[ |\mathbf{E}_{\mathrm{annotated}}'| \right] \le \frac{N_1}{20}.
\]
Applying Markov’s inequality at threshold $\frac{N_1}{4}-2$ gives
\[
\Pr\Big(|\mathbf{E}_{\mathrm{annotated}}'| > \frac{N_1}{4}-2\Big)
\le
\Pr\Big(|\mathbf{E}_{\mathrm{annotated}}'| > \frac{N_1}{4}-\frac{2N_1}{500}\Big)
\le \frac{25}{121}
< 0.21.
\]
Define event
\[
F_4 := \left\{ |\mathbf{E}_{\mathrm{annotated}}' | \le \frac{N_1}{4}-2 \right\}.
\]
Then $\Pr(F_4) \ge 0.79$.

\paragraph{Third}, we show that under events $F_3$ and $F_4$, the interactive phase of
\WarmStartDagger gets $\mathbf{b'}$ annotated with probability greater than 0.4 in each of the first three rollouts, such that
$\mathbf{b'}$ is annotated within three iterations with good probability.

First, we observe that by union bound,  
\[
\Pr(F_3, F_4) \geq 
\Pr(F_3) + \Pr(F_4) - 1
\geq 0.69. 
\]


We henceforth condition on $F_3 \cap F_4$ happening. 
By the definition of $F_3$, the policies $\pi^n$ produced by \WarmStartDagger acts optimally on all states in $\mathbf{E}$.

Denote by $\mathbf{E}_{\mathrm{annotated}}'^{,n}$ the set of annotated states in $\mathbf{E}$ after \WarmStartDagger's iteration $n$; with this notation,  $\mathbf{E}_{\mathrm{annotated}}'^{,0} = \mathbf{E}_{\mathrm{annotated}}'$, and $|\mathbf{E}_{\mathrm{annotated}}'^{,n+1}| \leq |\mathbf{E}_{\mathrm{annotated}}'^{,n}| + 1$ for all $n$. 
Thus, by the definition of $F_4$, for every $n = 0,1,2$, $\mathbf{E}_{\mathrm{annotated}}'^{,n+1} \leq \frac{N_1}{4} - 2 + 2 = \frac{N_1}{4}$.
In other words, 
at least $\frac{3}{4}$ of the states in $\mathbf{E'}$
are unannotated for each of the first three iterations. 


During the rollout of a $\pi^n$, 
let $\tau$ be the first step such that $s_\tau \in \mathbf{E'}$.
By Lemma~\ref{lem:S1-visit}, for every $n \in \cbr{1,2,3}$, 
\[
\Pr_{\Mcal, \pi^n}(\tau \le H/5) \ge 0.79.
\]

Since by the definition of $\Mcal$,  $s_\tau$ is drawn uniformly from $\mathbf{E'}$ conditioned on $\tau \leq H/5$, 
\[
\Pr_{\Mcal, \pi^n}(s_\tau \notin \mathbf{E}_{\mathrm{annotated}}'^{,n} \mid \tau \leq H/5) \ge 0.75.
\]
Therefore,
\begin{equation}
\Pr_{\Mcal, \pi^n}(s_\tau \notin \mathbf{E}_{\mathrm{annotated}}'^{,n}, \tau \leq H/5)
\geq 
0.79 \times 0.75 \geq 0.59
\label{eqn:see-unannotated-eprime}
\end{equation}

With $A \geq 10H$, by Lemma~\ref{lem:unannotated-wrong-general},
\begin{equation}
\Pr_{\Mcal, \pi^n}\big( \forall h \in [H]: 
s_h \in (\mathbf{E}' \setminus \mathbf{E}_{\mathrm{annotated}}'^{,n}) \cup \mathbf{B'} \implies
\ a_h \neq \pi^E(s_h) \big)
\;\ge\; 0.9.
\label{eqn:unannotated-wrong-wsd}
\end{equation}
Thus, when the events in Eqs.~\eqref{eqn:see-unannotated-eprime} and ~\eqref{eqn:unannotated-wrong-wsd} happen simultaneously (which happens with probability $\geq 0.59+0.9-1 \geq 0.49$ by union bound), 
the trajectory rolled out by $\pi^n$ transitions to $\mathbf{b'}$ at
step $\tau+1$ and stays there till the end of episode, accumulating no less than $0.8H$ $\mathbf{b'}$ states. 
Since \WarmStartDagger samples one state uniformly from each rollout for annotation, this gives 0.8 probability of $\mathbf{b'}$ being annotated with probability $\geq 0.49$.

Denote by $\Fcal_n$ the sigma-algebra generated by observations up to iteration $n$ of \WarmStartDagger. 
Our reasoning above implies that for all history $\Fcal_{n-1}$ such that $F_3 \cap F_4$ happens, for all $n \in \cbr{1,2,3}$, 
\[
\Pr( \mathbf{b'} \text{ is annotated at iteration $n$} \mid \Fcal_{n-1} )
\geq 0.49 \times 0.8 \geq 0.39.
\]

Define event
\[
F_5 := \{\mathbf{b'} \text{ is annotated within 3 iterations}\}.
\]



Hence
\[
\Pr(F_5 \mid F_3,F_4)
\;\ge\;
1 - (1 - 0.39)^3
= 1 - 0.61^3
> 0.77.
\]

Combining the three events, we have
\[
\Pr(F_3 \cap F_4 \cap F_5)
\;\ge\;
\Pr(F_3 \cap F_4)\,\Pr(F_5 \mid F_3,F_4)
\;\ge\;
0.69 \times 0.77
> 0.5.
\]

\paragraph{Finally}, on $F_3 \cap F_4 \cap F_5$, all states in $\mathbf{E}$ and the reset state
$\mathbf{b'}$ are annotated, under reward function $R_1$, the learner receives reward 1 at any step if it is in $\mathbf{E}$ or $\mathbf{E'}$ or it is in $\mathbf{b'}$ and takes the recovery action same as the expert. Since learned policy now behaves identical to $\pie$ on all states in $\mathbf{E}$ and can successfully recover in $\mathbf{B'}$, its total return is the same as the expert, which concludes the proof.
\end{proof}

\begin{lemma}[Coverage of states in $\mathbf{E}$ with $N_\mathrm{off}$ trajectories]
\label{lem:E-expert-coverage-bound}
Consider the MDP $\Mcal$. 
For $\Noff$ trajectories of length $H$ rolled out by expert policy $\pie$,
all $N_0$ states in $\mathbf{E}$ are annotated with probability $\geq 1-\delta$ if:
    \[
\Noff \;\ge\; \frac{N_0}{(1-\beta)H} \log\!\Big(\frac{N_0}{\delta}\Big).
\]
\end{lemma}


\begin{proof}


Recall that we denote by $\mathbf{E}_{\mathrm{annotated}}$ the set of states in $\mathbf{E}$ visited and annotated by offline expert demonstrations. 
Fix any $s \in \mathbf{E}$ and concatenate all expert trajectories into a single sequence
$\{s_t\}_{t=1}^{T}$ of length $T := H \Noff$, ordered from the first state of the
first trajectory to the last state of the last trajectory. Let $\Fcal_t$ be the sigma-algebra generated by all states $s_1,\dots,s_t$, and define the set of initial indices
\[
I_0 := \{\,1,\; H+1,\; 2H+1,\; \dots,\; (\Noff-1)H+1\,\}.
\]

For any $t+1 \in I_0$, $s_{t+1}$ is drawn from $\rho$, so
\[
\Pr(s_{t+1}=s \mid \Fcal_t) = \rho(s) = \frac{1}{(1+\beta)N_0}.
\]
For $t+1 \notin I_0$, we have $s_t \in \mathbf{E}\cup\mathbf{E'}$ and under $\pi^E$:
\[
\Pr(s_{t+1}=s \mid s_t) =
\begin{cases}
\dfrac{1-\beta}{N_0}, & s_t \in \mathbf{E},\\[2pt]
\dfrac{1}{N_0}, & s_t \in \mathbf{E'}.
\end{cases}
\]
Therefore, in all cases,
\[
\Pr(s_{t+1}=s \mid \Fcal_t) \;\ge\; \frac{1-\beta}{N_0},
\qquad
\Pr(s_{t+1}\neq s \mid \Fcal_t) \;\le\; 1 - \frac{1-\beta}{N_0}.
\]

Let $A_t := \{s_1\neq s,\dots,s_t \neq s\}$. Then
\[
\Pr(A_{t+1})
\le \Big(1-\frac{1-\beta}{N_0}\Big)\Pr(A_t),
\]
Thus, 
\[
\Pr\big(s \notin \mathbf{E}_{\mathrm{annotated}}\big)
= 
\Pr(A_T)
\le \Big(1-\frac{1-\beta}{N_0}\Big)^T
\le \exp\!\Big(-\frac{(1-\beta)T}{N_0}\Big)
= \exp\!\Big(-\frac{(1-\beta)H \Noff}{N_0}\Big).
\]

By a union bound over all $s \in \mathbf{E}$,
\[
\Pr(\exists s\in \mathbf{E}, s \notin \mathbf{E}_{\mathrm{annotated}})
\le N_0 \exp\!\Big(-\frac{(1-\beta) H \Noff}{N_0}\Big),
\]
which is at most $\delta$ whenever
\[
\Noff \;\ge\; \frac{N_0}{(1-\beta)H} \log\!\Big(\frac{N_0}{\delta}\Big).
\qedhere
\]

\end{proof}

\section{Additional Guarantees for a Trajectory-wise
DAgger Variant without Recoverability Assumption}
\label{sec:decoupled_h_dist}

In this section, we revisit and conduct a refined analysis of another variant of DAgger with trajectory-wise annotations. We show that without the recoverability 
assumption, an interactive IL algorithm has sample complexity no worse than that of behavior cloning. This result complements~\cite{foster2024behavior} that analyzes a different version of trajectory-wise DAgger, which they proved to have a worse sample complexity guarantee than behavior cloning.

\subsection{Additional Notations and Useful Distance Measures}
\label{sec:distance_measures}

Recall that we have defined first-step mixing of policies in Section~\ref{sec:first-step-mixing-causal}.
We instantiate it to define a mixture of policies in $\Bcal$, which induces a useful policy class below:

\begin{definition}[First-step mixing of $\Bcal$]
    Define $\Pi_{\Bcal} := \cbr{ \pi_u  : u \in \Delta(\Bcal) }$,
where policy $\pi_u$ is executed in an an episode of an MDP $\Mcal$ by: draw $\pi \sim u$ at the beginning of the episode, and execute policy $\pi$ throughout the episode. 
   \label{def:mixed_class_first_step} 
\end{definition}

In the following, we present another two useful distance measures for a pair of policies.

\begin{definition}[Trajectory-wise $L_\infty$-semi-metric~\cite{foster2024behavior}]
    For a pair of Markovian policies $\pi$ and $\pi'$, define their trajectory-wise $L_\infty$-semi-metric as
    \[
\rho(\pi\parallel \pi') := \mathbb{E}^{\pi} \mathbb{E}_{a'_{1:H} \sim \pi'(\cdot \| s_{1:H})} \left[ \mathbb{I} \left\{ \exists h : a_h \neq a'_h \right\} \right].
\]
\label{def:traj_l_inf}
\end{definition}
$\rho( \pi\parallel \pi')$ is the probability of any action taken by $\pi'$ deviating from actions in trajectories induced by $\pi$, which is symmetric~\cite{foster2024behavior}. 
A bound on $\rho(\pi\parallel \pie)$ leads to straightforward performance difference guarantee: $J(\pie) - J(\pi) \leq R\cdot\rho(\pi\parallel \pie)$~\cite{foster2024behavior} (Lemma~\ref{lem:foster_reduction}).

Recall that we have defined causually conditional probabilities in Section~\ref{sec:first-step-mixing-causal}. Built upon it, we introduce the following definition:

\begin{definition}[Decoupled Hellinger distance]
For a pair of Markovian policies $\pi$ and $\pi'$, define their decoupled Hellinger distance as
$ \mathbb{E}^{\pi} \sbr{D^2_\text{H}(\pi(\cdot \| s_{1:H}),\pi'(\cdot \| s_{1:H}))}$.
\label{def:traj_h_dist}
\end{definition}
Similarly, 
$ \mathbb{E}^{\pi} \sbr{D^2_\text{H}(\pi(\cdot \| s_{1:H}),\pi'(\cdot \| s_{1:H}))}$ denotes the expected Hellinger distance between the causal distribution of actions $\pi(\cdot \| s_{1:H})$ and $\pi'(\cdot \| s_{1:H})$ on state sequence $s_{1:H}$ visited by $\pi$. This allows decoupled analysis for state and action sequences, which is useful for 
our analysis below.

\subsection{Interactive IL Matches Offline IL on Trajectory-wise Annotation Sample Complexity}
We present $\TrajDagger$ (Algorithm~\ref{alg:dagger_step_one}), another variant of DAgger   that operates in the trajectory-wise sampling model, 
and provide its sample complexity bounds.

\begin{algorithm}
\caption{$\TrajDagger$: DAgger with trajectory-wise annotation oracle}
\label{alg:dagger_step_one}
\begin{algorithmic}[1]
\STATE \textbf{Input:} MDP $\Mcal$, deterministic expert $\pie$, stationary policy class $\Bcal$, online learning oracle $\alg$ with decision space $\Pi_\Bcal$.
\FOR{$n = 1, \dots, N$}
    \STATE Query $\alg$ and receive $\pi^n \in \Pi_\Bcal$. 
    
    \STATE Roll out $\pi^n$ and sample $s^n_{1:H}$ following $\PP^{\pi^n}$. Query $\Otraj$ for $a_{1:H}^{*,n}=\pie(s^n_{1:H})$.
    \STATE Update $\alg$ with loss function
    \begin{equation}
    \ell^n(\pi)
    := \log\rbr{\frac1{
    \pi(a_{1:H}^{*,n} \parallel s_{1:H}^n)}}
    .
    \label{eqn:ell-n-u-2}
    \end{equation}
    
    
\ENDFOR
\STATE Output $\hat{\pi}$, the first-step uniform mixture of policies in $\cbr{\pi^1, \ldots, \pi^n}$.
\end{algorithmic}
\end{algorithm}

Algorithm~\ref{alg:dagger_step_one} uses first-step mixing policies $\pi_u \in \Pi_{\Bcal}$ (recall Definition~\ref{def:mixed_class_first_step}). At round $n$, it rolls out $\pi^n = \pi_{u^n}$ obtained from an online learning oracle $\alg$ and samples a full state sequence $s^n_{1:H}$. 
Similar to Algorithm~\ref{alg:dagger_one_sample}, Algorithm~\ref{alg:dagger_step_one} also requires $\alg$ to have decision space 
$\Pi_\Bcal$ (cf. $\bar{\Pi}_\Bcal$ in Algorithm~\ref{alg:dagger_one_sample}).
It then requests expert's trajectory-wise annotation $a^{*,n}_{1:H}$ and updates $\alg$ by $\ell^n(\pi)$ (Eq.~\eqref{eqn:ell-n-u-2}). 
At the end of iteration $N$, the uniform first-step mixing of $\{\pi^n\}_{n=1}^N$ is returned, which is equivalent to 
%
returning $\pi_{\hat{u}}$, where $\hat{u} := \frac{1}{N}\sum_{n=1}^N u^n$.
We provide the following performance guarantee of Algorithm~\ref{alg:dagger_step_one}:


\begin{theorem}
\label{thm:log_loss_dagger_step_1}
If Algorithm~\ref{alg:dagger_step_one} is run with a deterministic expert policy $\pi^E$, a policy class $\Bcal$ such that realizability holds, and the online learning oracle $\alg$ set as the exponential weight algorithm, then it returns $\hat{\pi}$ such that, with probability at least $1-\delta$,
\[
J(\pie) - J(\hat{\pi}) 
\leq 
2R \frac{ \log(B) + 2\log(1/\delta)}{N}.
\]
\end{theorem}

Theorem~\ref{thm:log_loss_dagger_step_1} shows that in the deterministic realizable setting,
the interactive IL Algorithm~\ref{alg:dagger_step_one} has a trajectory-wise sample complexity matching that of behavior cloning~\cite{foster2024behavior}. 
In contrast, prior state-of-the-art analysis of interactive IL algorithms~\cite[][Appendix C.2]{foster2024behavior} gives sample complexity results that are in general worse than behavior cloning. \footnote{For~\cite[][Appendix C.2]{foster2024behavior}'s sample complexity to improve over behavior cloning, we need $\mu H \max_{h \in [H]} \log |\Bcal_h|$ to be significantly smaller $R \log|\Bcal|$ (where $\Bcal_h$ is the projection of $\Bcal$ onto step $h$). This may require the strong condition that $\mu \ll R / H \leq 1$ in the more practical parameter-sharing settings ($|\Bcal_h| = |\Bcal|$). 
}



For the proof of Theorem~\ref{thm:log_loss_dagger_step_1}, we introduce a new notion called decoupled Hellinger estimation error: \[
\Ontraj_N :=  \sum_{n=1}^{N} \mathbb{E}^{\pi^n} \sbr{D^2_\text{H}(\pi^n(\cdot \| s_{1:H}),\pie(\cdot \| s_{1:H}))}.
\]
$\Ontraj_N$ decouples the dependence between the state sequence and the distribution of action sequence induced by the learner. Perhaps surprisingly, as we show below, it is compatible with non-Markovian first-step mixing of policies, while still being well-behaved enough to be translated to a policy suboptimality guarantee, which could be of independent interest.

\subsection{Decoupling State and Action Sequences by Decoupled Hellinger Distance}


In this section, we demonstrate that
similar to $D^2_{\text{H}}\rbr{\mathbb{P}^{\pi},\mathbb{P}^{\pie}  }$~\cite{foster2024behavior}, 
the decoupled Hellinger distance
$\mathbb{E}^{\pi} \sbr{D^2_\text{H}(\pi(\cdot \| s_{1:H}),\pie(\cdot \| s_{1:H})) } $ 
is also proportionally lower bounded by a constant factor of $\rho(\pi\parallel \pie)$.
The following two lemmas show that such relationship holds for both Markovian policies and their first-step mixings.


\begin{lemma}
Let $\pie$ be a deterministic policy, and let $\pi$ be a Markovian policy. Then we have
\[
\frac{1}{2} \cdot \rho(\pi \parallel \pie) 
\leq 
\mathbb{E}^{\pi} \sbr{D^2_\text{H}(\pi(\cdot \| s_{1:H}),\pie(\cdot \| s_{1:H})) }.
\]
\label{lem:decoupled_hellinger_bound_markovian}
\end{lemma}

\begin{lemma}
Let $\pie$ be a deterministic policy, and let $\pi_u$ be a first-step mixing of Markovian policies. Then we have

\[
\frac{1}{2} \cdot \rho(\pi_u \parallel \pie) 
\leq 
\mathbb{E}^{\pi_u} \sbr{D^2_\text{H}(\pi_u(\cdot \| s_{1:H}),\pie(\cdot \| s_{1:H})) }.
\]
\label{lem:decoupled_hellinger_bound_general}
\end{lemma}

The cornerstone of the above two lemmas is the following special case about first-step mixing of deterministic policies. 
Given an MDP with finite state space size $S$ and action space size $A$, we denote the set of all deterministic, Markovian policies, as $\pidet$. 
$\pidet$ 
contains $A^{S H}$ deterministic policies, which can be indexed by a tuple of actions $(a_{h,s})_{h \in [H],\, s \in \mathcal{S}}$. 



\begin{lemma}
\label{lem:decoupled_hellinger_bound_base_case}
Let $\pie$ be a deterministic Markovian policy, and let $\pi_u$ 
be a first-step mixing of deterministic Markovian policies (i.e., elements of $\pidet$). Then we have that
\[
\frac{1}{2} \cdot \rho(\pi_u \parallel \pie) 
\leq 
\mathbb{E}^{\pi_u} \sbr{D^2_\text{H}(\pi_u(\cdot \| s_{1:H}),\pie(\cdot \| s_{1:H})) }.
\]
\end{lemma}

We now quickly conclude Lemmas~\ref{lem:decoupled_hellinger_bound_markovian} and~\ref{lem:decoupled_hellinger_bound_general} using Lemma~\ref{lem:decoupled_hellinger_bound_base_case} in the next subsection, then come back to the proof of Lemma~\ref{lem:decoupled_hellinger_bound_base_case} in the subsection after.

\subsubsection{Proofs of Lemmas~\ref{lem:decoupled_hellinger_bound_markovian} and~\ref{lem:decoupled_hellinger_bound_general}
}
\label{sec:decoupled_hellinger_bound_markovian}


\begin{proof}[Proof of Lemma~\ref{lem:decoupled_hellinger_bound_markovian}]
We show the following simple claim: any Markovian policy $\pi$ is equivalent to a first-step mixing of a set of deterministic Markovian policies.  This allows us to apply guarantees for mixtures of deterministic Markovian policies in Lemma~\ref{lem:decoupled_hellinger_bound_base_case} to conclude the proof.
\end{proof}

\begin{claim}
\label{claim:first_step_mixture_equivalence}
For a Markovian policy $\pi = (\pi_1, \ldots, \pi_H)$, there exists a first-step mixing of deterministic policy $\pi_u$ such that for any $s_{1:H} \in \Scal^H$, \textit{1.} $\pi( \cdot \| s_{1:H} ) = \pi_u(\cdot \| s_{1:H} )$, and \textit{2.} $\PP^{\pi}( s_{1:H} ) = \PP^{\pi_u} (s_{1:H})$.
\end{claim}

\begin{proof}[Proof of Claim~\ref{claim:first_step_mixture_equivalence}]
To construct policy $\pi_u$,  we will set the weight vector $u$ (over $\pidet$)
such that its weight on policy $\nu$ indexed by $(a_{h,s})_{h \in [H], s \in \Scal}$ as: 
\begin{equation}
u(\nu) = \prod_{h=1}^H \prod_{s \in \Scal} \pi_h(a_{h,s} | s)
\end{equation}
It can be easily verified that 
$\sum_{\nu \in \pidet} u(\nu) = 1$. 


We now verify the first item.
By first-step mixing, we rewrite $\pi_u( a_{1:H} \parallel s_{1:H} )$ 
as


\begin{equation}
    \begin{aligned}
        \pi_u( a_{1:H} \parallel s_{1:H} )
        = &
        \sum_{\nu\in \pidet} u(\nu) \prod_{h=1}^{H}  \nu_{h}(a_{h} | s_{h})\\
        = &
        \sum_{(a'_{h,s})_{h \in [H], s \in \Scal}}\prod_{h=1}^H \prod_{s \in \Scal} \pi_h(a'_{h,s} | s)\prod_{h=1}^H\mathbb{I}\sbr{a'_{h,s_h} = a_h }\\
        = &
        \sum_{(a'_{h,s})_{h \in [H], s \neq s_h}}\prod_{h=1}^H \prod_{s \neq s_h} \pi_h(a'_{h,s} | s)
        \sum_{(a'_{h,s})_{h \in [H], s = s_h}}\prod_{h=1}^H \pi_h(a'_{h,s_h} | s_h)
        \prod_{h=1}^H\mathbb{I}\sbr{a'_{h,s_h} = a_h }\\
        = &
        \sum_{(a'_{h,s})_{h \in [H], s \neq s_h}}\prod_{h=1}^H \prod_{s \neq s_h} \pi_h(a'_{h,s} | s)
        \prod_{h=1}^H \pi_h(a_{h} | s_{h})\\
        = &
        \prod_{h=1}^H \pi_h(a_{h} | s_{h})
        =
        \pi( a_{1:H} \parallel s_{1:H} ).
    \end{aligned} \label{eq:mixture_action_distribution}
\end{equation}

Since this holds for any action sequence $a_{1:H} \in \Acal^H$, we derive the first part of Claim~\ref{claim:first_step_mixture_equivalence} that  $ \pi(\cdot \| s_{1:H}) =  \pi_u(\cdot \| s_{1:H} )$.
The second item follows from the first item combined with Lemma~\ref{lem:prob-fact-causal}. 
\end{proof}







\begin{proof}[Proof of Lemma~\ref{lem:decoupled_hellinger_bound_general}]
By Claim~\ref{claim:first_step_mixture_equivalence},
any Markovian policy can be viewed as a first-step mixing of $A^{SH}$ deterministic policies from $\pidet$, thus, any first-step mixing of Markovian policies $\pi_u$ can also be viewed as a first-step mixing of $A^{SH}$ deterministic policies from $\pidet$. The proof again follows by applying Lemma~\ref{lem:decoupled_hellinger_bound_base_case}.
\end{proof}


\subsubsection{Proof of Lemma~\ref{lem:decoupled_hellinger_bound_base_case}}

To facilitate the proof of Lemma~\ref{lem:decoupled_hellinger_bound_base_case}, we introduce the following additional notations:
\begin{itemize}
    \item Recall that $\pidet$ 
    is the set of all deterministic, Markovian policies.
    We will use $\nu, \nu'$ to denote members of $\pidet$ and $\nu_h(s)$ to denote the action $\nu$ takes at state $s$ at step $h$ when it is clear from the context.

    \item Let $\eclass(s_{1:h})$ represent the subset of $\pidet$ that agrees with $\pie$ on the state sequence $s_{1:h}$.
    \item Define $
    F(\nu; \nu'; \pie) := \sum_{s_{1:H}} 
    \mathbb{P}^{\nu} (s_{1:H}) 
    \mathbb{I} \left[ \nu' \notin \eclass(s_{1:H})\right],$
    which evaluates the probability that $\nu'$ disagrees with $\pie$ over the distribution of $H$-step state sequences induced by $\pi$.
\end{itemize}




Our key idea is to lower bound 
$\mathbb{E}^{\pi_u} \sbr{D^2_\text{H}(\pi_u(\cdot \| s_{1:H}),\pie(\cdot \| s_{1:H})) },$
which 
reflects the asymmetric roles of the two appearances of $\pi_u$'s,
using a symmetric formulation via function \(F\) (as shown in~\eqref{eq:symmetric_representation} below).


\begin{proof}
Recall the first-step mixing policy in Definition~\ref{def:mixed_class_first_step}, 
we start by rewriting




\begin{equation}
    \begin{aligned}
        \rho(\pi_u \parallel \pie) 
        =&
        \mathbb{E}^{\pi_u} \left[ 
        \mathbb{I} \left\{ \exists h : a_h \neq \pie_h(s_h)
        \right\} \right]\\
        =&
        \sum_{\nu\in \pidet} u(\nu) \sum_{s_{1:H}}\mathbb{P}^\nu(s_{1:H},a_{1:H}) \mathbb{I}\left\{ \exists h : a_h \neq \pie_h(s_h) \right\} \\
        =&
        \sum_{\nu\in \pidet} u(\nu) \rho(\nu\parallel \pie) ,
    \end{aligned}
\end{equation}
which is a weighted combination of $\rho(\nu\parallel \pie)$ for $\nu\in \pidet $.


Next, we turn to analyzing $D^2_\text{H}(\pi_u(\cdot \| s_{1:H}),\pie(\cdot \| s_{1:H}))$. Since the deterministic expert induces a delta mass distribution over actions, we apply the elementary fact about the Hellinger distance with delta mass distribution stated in Lemma~\ref{lem:hellinger-pointmass}, yielding:
\[
\frac{1}{2}\parallel\pi_u(\cdot \| s_{1:H}) - \pie(\cdot \| s_{1:H})\parallel_1
\leq
D^2_\text{H}(\pi_u(\cdot \| s_{1:H}),\pie(\cdot \| s_{1:H})).
\] 
We recall that $\eclass(s_{1:H})$ denotes the subset of $\pidet$ that agrees with $\pie$ on $s_{1:H}$ 
and define the total weight assigned by $u$ on it as $u(\eclass(s_{1:H})) := \sum_{\nu\in \eclass(s_{1:H})} u(\nu)$. 
Then, 
\[
\frac12\parallel\pi_u(\cdot \|s_{1:H}) - \pie(\cdot \|s_{1:H})\parallel_1
=
1 - u(\eclass(s_{1:H})),
\] 
which implies:
\begin{equation}
1 - u(\eclass(s_{1:H}))
\leq
D^2_\text{H}(\pi_u(\cdot \| s_{1:H}),\pie(\cdot \| s_{1:H})).
\label{eqn:d-h-sq-lb}
\end{equation}

Therefore, by taking expectation over $s_{1:H} \sim \PP^{\pi_u}$ in Eq.~\eqref{eqn:d-h-sq-lb},


\begin{equation}
\sum_{s_{1:H}} \mathbb{P}^{\pi_u} (s_{1:H}) (1 - u(\eclass(s_{1:H})))
\leq
\EE^{\pi_u}\sbr{D^2_\text{H}(\pi_u(\cdot \| s_{1:H}),\pie(\cdot \| s_{1:H})) }.
\label{eqn:sandwich-pi-u}
\end{equation}

We now examine the expression 
\begin{equation}
\sum_{s_{1:H}} \mathbb{P}^{\pi_u} (s_{1:H}) (1 - u(\eclass(s_{1:H})))
\tag{*}.
\end{equation}
Since $\pi_u$ is a first-step mixing of policies in $\pidet$ with weight $u$, we have $\mathbb{P}^{\pi_u} (s_{1:H}) = \sum_{\nu\in \pidet} u(\nu)\mathbb{P}^\nu(s_{1:H})$. This allows us to rewrite $(*)$ using the 
definition of $F(\nu; \nu', \pie)$ as:

\begin{equation}
    \begin{aligned}
        (*)
        = &
        \sum_{s_{1:H}} \sum_{\nu\in \pidet} u(\nu)\mathbb{P}^\nu(s_{1:H}) \sum_{\nu' \in \pidet}u(\nu')\mathbb{I} \left[ \nu' \notin \eclass(s_{1:H})\right] \\
        = &
        \sum_{\nu, \nu' \in \pidet} 
        u(\nu) u(\nu') \sum_{s_{1:H}}\mathbb{P}^\nu(s_{1:H}) 
        \mathbb{I} \left[ \nu' \notin \eclass(s_{1:H})\right] \\
        = &
        \sum_{\nu, \nu' \in \pidet}  u(\nu) u(\nu') F(\nu; \nu'; \pi^E)
        \\
        = &
        \frac12
        \sum_{\nu, \nu' \in \pidet}  u(\nu) u(\nu') \rbr{ F(\nu; \nu'; \pi^E) +  F(\nu; \nu'; \pi^E) },
    \end{aligned}
    \label{eq:symmetric_representation}
\end{equation}
where the first three equalities are by algebra and the definition of $F(\nu; \nu; \pi^E)$.
In the last equality, we use the observation that 
\[
\sum_{\nu, \nu' \in \pidet} u(\nu) u(\nu') F(\nu; \nu'; \pi^E) =  \sum_{\nu, \nu' \in \pidet} u(\nu) u(\nu') F(\nu'; \nu; \pi^E).
\]




By Lemma~\ref{lem:decoupled_foundamental} (stated below), 
\begin{align*}
(*) \geq & 
  \frac12 \cdot \frac12 \cdot  
        \sum_{\nu, \nu' \in \pidet} u(\nu) u(\nu') \rbr{ \rho(\nu\parallel \pie) + 
\rho(\nu' \parallel \pie ) } \\
= &  \frac12 \cdot \sum_{\nu\in \pidet} u(\nu) \rho(\nu\parallel \pie)
=
\frac12 \cdot \rho(\pi_u \parallel \pie).
\end{align*}

Combining the above two inequalities with Eq~\eqref{eqn:sandwich-pi-u} we conclude the proof by 
\[
\frac12 \cdot \rho(\pi_u \parallel \pie)
\leq
(*)
\leq
\EE^{\pi_u}\sbr{D^2_\text{H}(\pi_u(\cdot \| s_{1:H}),\pie(\cdot \| s_{1:H})) }.
\qedhere
\]
\end{proof}

\begin{lemma}[Symmetric Evaluation Lemma]
Given deterministic Markovian policies $\nu$, $\nu'$, and $\pie$, 
the following holds:
\begin{equation}
\begin{aligned} 
\frac12 \cdot 
\rbr{\rho(\nu\parallel \pie) + \rho(\nu' \parallel \pie)}
\leq 
F(\nu; \nu'; \pie) + 
F(\nu'; \nu; \pie).
\end{aligned}
\end{equation}
\label{lem:decoupled_foundamental}
\end{lemma}



\begin{proof}
Recall that 
\[
F(\nu; \nu'; \pie) + 
F(\nu'; \nu; \pie)
= 
\sum_{s_{1:H}} 
\rbr{
\mathbb{P}^\nu(s_{1:H}) \mathbb{I} \left[ \nu' \notin \eclass(s_{1:H})\right] 
        +
        \mathbb{P}^{\nu'} (s_{1:H}) \mathbb{I} \left[ \nu \notin \eclass(s_{1:H})\right]
        }.
\]
Throughout the proof, we say that $\nu$ makes a \textit{mistake} at step $h$,
if $\nu_h(s_h) \neq \pie_h(s_h)$.
Then, we can partition all state sequences $s_{1:H} \in \Scal^{H}$ into 4 subsets, $\Xcal_i$, indexed by $i \in \{1,2,3,4\}$:
\begin{enumerate} 
    \item $\Xcal_1 := \{s_{1:H} \; | \; \nu,\nu' \in \eclass(s_{1:H})\}$;
    \item $\Xcal_2 := \{s_{1:H} \; | \; \exists h , s.t.  \nu\in \eclass(s_{1:h}) ,\nu' \notin \eclass(s_{1:h}), \nu' \in \eclass(s_{1:h-1}) \}$;
    \item $\Xcal_3 := \{s_{1:H} \; | \; \exists h , s.t.  \nu\notin \eclass(s_{1:h}) ,\nu' \in \eclass(s_{1:h}), \nu\in \eclass(s_{1:h-1}) \}$;
    \item $\Xcal_4 := \{s_{1:H} \; | \; \exists h , s.t.  \nu\notin \eclass(s_{1:h}) ,\nu' \notin \eclass(s_{1:h}), \nu\in \eclass(s_{1:h-1}) ,\nu' \in \eclass(s_{1:h-1})\}  \}$.
\end{enumerate}


In words, the four subsets divide state sequences into cases where: (1) both $\nu,\nu'$ agree with the $\pie$ throughout, (2)$\&$(3) one of $\nu, \nu'$ makes its first mistake
earlier than the other, and (4) $\nu,\nu'$ make their first mistake at the same time. 
It can now be easily seen that
each $s_{1:H} \in \Scal^{H}$ lies in exactly one of such $\Xcal_i$, and 
\[
\Xcal_1 \cup \Xcal_2 \cup \Xcal_3 \cup \Xcal_4
 =
 \Scal^H.
\]
To see this, consider $h^{\text{err}}$, the first time step $h$ such that one of $\nu$ and $\nu'$ disagree with $\pie$. If $h^{\text{err}}$ does not exist, then 
$s_{1:H} \in \Xcal_1$.
Otherwise, $s_{1:H}$ lies in one of $\Xcal_2, \Xcal_3, \Xcal_4$ depending on whether $\nu$ and $\nu'$ makes mistakes at step $h^{\text{err}}$.

By definition, subset $\Xcal_1$ denotes trajectories $s_{1:H}$
where $\nu,\nu' \in \eclass(s_{1:H})$, meaning that
\[
\sum_{s_{1:H} \in \Xcal_1} 
        \rbr{
        \mathbb{P}^\nu(s_{1:H}) \mathbb{I} \left[ \nu' \notin \eclass(s_{1:H})\right] 
        +
        \mathbb{P}^{\nu'} (s_{1:H}) \mathbb{I} \left[ \nu \notin \eclass(s_{1:H})\right]
        } = 0 \;.
\]

For the other 3 sets, i.e. $\Xcal_i$ for $i \in \{2,3,4\}$, we can further divide each set based on the time step where the first error occurs, formally:

\begin{equation}
    \begin{aligned}
        \Xcal_2^h := & \{s_{1:H} \; | \; \nu\in \eclass(s_{1:h}) ,\nu' \notin \eclass(s_{1:h}), \nu' \in \eclass(s_{1:h-1}) \};\\
        \Xcal_3^h := & \{s_{1:H} \; | \; \nu\notin \eclass(s_{1:h}) ,\nu' \in \eclass(s_{1:h}), \nu\in \eclass(s_{1:h-1}) \};\\
        \Xcal_4^h := & \{s_{1:H} \; | \; \nu\notin \eclass(s_{1:h}) ,\nu' \notin \eclass(s_{1:h}), \nu\in \eclass(s_{1:h-1}), \nu' \in \eclass(s_{1:h-1}) \}.
    \end{aligned}
\end{equation}

By definition, each pair of subsets is disjoint and
%
$\cup_{h \in [H]} \Xcal_i^h = \Xcal_i$, for $i=2,3,4$.
Note that 
the determination of whether $s_{1:H} \in \Xcal_i^h$ only depends on $s_{1:h}$; therefore, $\Xcal_i^h$ can be represented as $\tilde{\Xcal}_i^h \times \Scal^{H-h}$, where
\[
\tilde{\Xcal}_i^h := \{s_{1:h} \; | \; s_{1:H} \in  \Xcal_i^h \}.
\]


Based on this observation, we have
\[
\sum_{s_{1:H} \in \Xcal_i^h} \mathbb{P}^\nu(s_{1:H})
=
\sum_{s_{1:h} \in \tilde{\Xcal}_i^h, s_{h+1:H} \in \Scal^{H-h}} \mathbb{P}^\nu(s_{1:H})
=
\sum_{s_{1:h} \in \tilde{\Xcal}_i^h} \mathbb{P}^\nu(s_{1:h}).
\]

Furthermore, since deterministic policies $\nu,\nu',\pie$ agrees with each other for all $\{s_{1:h-1} | s_{1:h} \in  \tilde{\Xcal}_i^h \}$,


\begin{equation}
    \begin{aligned}
        \sum_{s_{1:h} \in \tilde{\Xcal}_i^h} \mathbb{P}^\nu(s_{1:h})
        =&
        \sum_{s_{1:h} \in \tilde{\Xcal}_i^h} \rho(s_1) \prod_{h'=1}^{h-1} P_{h}(s_{h+1} | s_{h}, \nu_h(s_h)) \\ 
        = &
        \sum_{s_{1:h} \in \tilde{\Xcal}_i^h} \rho(s_1) \prod_{h'=1}^{h-1} P_{h}(s_{h+1} | s_{h}, \nu'_h(s_h)) 
        =
        \sum_{s_{1:h} \in \tilde{\Xcal}_i^h} \mathbb{P}^{\nu'} (s_{1:h}).
    \end{aligned}     
\end{equation}


This implies that 
\[
\sum_{s_{1:H} \in \Xcal_i^h} \mathbb{P}^\nu(s_{1:H})
= 
\sum_{s_{1:H} \in \Xcal_i^h} \mathbb{P}^{\nu'} (s_{1:H}),
\]
%
and therefore, summing over all $h \in [H]$,
\[
\sum_{s_{1:H} \in \Xcal_i} \mathbb{P}^{\nu} (s_{1:H})
=
\sum_{s_{1:H} \in \Xcal_i} \mathbb{P}^{\nu'} (s_{1:H}).
\]





Now, for $\Xcal_2$, we have
\begin{equation}
    \begin{aligned}
        &\sum_{s_{1:H} \in \Xcal_2} 
        \rbr{
        \mathbb{P}^\nu(s_{1:H}) \mathbb{I} \left[ \nu' \notin \eclass(s_{1:H})\right]
        +
        \mathbb{P}^{\nu'} (s_{1:H}) \mathbb{I} \left[ \nu \notin \eclass(s_{1:H})\right]
        } 
        %
        \geq
        \sum_{s_{1:H} \in \Xcal_2} 
        \mathbb{P}^\nu(s_{1:H}) \;,
    \end{aligned}
    \label{eqn:x2}
\end{equation}
where we apply the fact that for all $s_{1:H} \in \Xcal_2$, $\nu' \notin \eclass(s_{1:H})$, and dropping the second term which is nonnegative.

Similarly, for $\Xcal_3$, we have that
\begin{equation}
    \begin{aligned}
        &\sum_{s_{1:H} \in \Xcal_3} 
        \rbr{
        \mathbb{P}^\nu(s_{1:H}) \mathbb{I} \left[ \nu' \notin \eclass(s_{1:H})\right]
        +
        \mathbb{P}^{\nu'} (s_{1:H}) \mathbb{I} \left[ \nu \notin \eclass(s_{1:H})\right]
        } 
        \geq 
        \sum_{s_{1:H} \in \Xcal_3} 
        \mathbb{P}^{\nu'} (s_{1:H}) 
        =
         \sum_{s_{1:H} \in \Xcal_3} 
        \mathbb{P}^\nu(s_{1:H}).
    \end{aligned}
    \label{eqn:x3}
\end{equation}
Finally, for $\Xcal_4$, we 
use the fact that for $s_{1:H} \in \Xcal_4$, $\nu, \nu' \notin \eclass(s_{1:H})$ and obtain
\begin{equation}
    \begin{aligned}
        &\sum_{s_{1:H} \in \Xcal_4} 
        \rbr{
        \mathbb{P}^\nu(s_{1:H}) \mathbb{I} \left[ \nu' \notin \eclass(s_{1:H})\right] 
        +
        \mathbb{P}^{\nu'} (s_{1:H}) \mathbb{I} \left[ \nu \notin \eclass(s_{1:H})\right]
        } \\
        =& 
        \sum_{s_{1:H} \in \Xcal_4} (\mathbb{P}^\nu(s_{1:H}) + \mathbb{P}^{\nu'} (s_{1:H}))
        \geq
        \sum_{s_{1:H} \in \Xcal_4} \mathbb{P}^\nu(s_{1:H}).
    \end{aligned}
    \label{eqn:x4}
\end{equation}

Now, we combine 
Eqs.~\eqref{eqn:x2},~\eqref{eqn:x3},~\eqref{eqn:x4}
and observe that 

\begin{equation}
    \begin{aligned}
        \sum_{s_{1:H} \in \Xcal_2} \mathbb{P}^\nu(s_{1:H})
        +
        \sum_{s_{1:H} \in \Xcal_3} \mathbb{P}^\nu(s_{1:H})
        +
        \sum_{s_{1:H} \in \Xcal_4} \mathbb{P}^\nu(s_{1:H})
        \geq &
        \frac{1}{2}\sum_{s_{1:H} \in \Xcal_2 \cup \Xcal_3 \cup \Xcal_4} \rbr{\mathbb{P}^\nu(s_{1:H})+\mathbb{P}^{\nu'}(s_{1:H})},
    \end{aligned}
\end{equation}
which implies 


\begin{equation}
    F(\nu; \nu'; \pie) + 
    F(\nu'; \nu; \pie)
    \geq
    \frac{1}{2} \sum_{s_{1:H} \in \Xcal_2 \cup \Xcal_3 \cup \Xcal_4} \rbr{\mathbb{P}^\nu(s_{1:H})+\mathbb{P}^{\nu'}(s_{1:H})}.
    \label{eq:symmetric_intermediate}
\end{equation}


Based on the definitions of $\Xcal_2, \Xcal_3, \Xcal_4$ and $\rho(\cdot \parallel \cdot)$,
\begin{equation}
    \begin{aligned}
         \sum_{s_{1:H} \in \Xcal_2 \cup \Xcal_3 \cup \Xcal_4} \rbr{\mathbb{P}^\nu(s_{1:H})+\mathbb{P}^{\nu'}(s_{1:H})}
        =&
        \sum_{s_{1:H}}\mathbb{P}^\nu(s_{1:H}) \mathbb{I}\left\{ \exists h : \nu_h(s_h) \neq \pie_h(s_h) \; \text{or} \; \nu'_h(s_h) \neq \pie_h(s_h) \right\}\\
        &+
        \sum_{s_{1:H}}\mathbb{P}^{\nu'}(s_{1:H}) \mathbb{I}\left\{ \exists h : \nu_h(s_h) \neq \pie_h(s_h) \; \text{or} \; \nu'_h(s_h) \neq \pie_h(s_h) \right\}\\
        \geq &
        \sum_{s_{1:H}}\mathbb{P}^\nu(s_{1:H}) \mathbb{I}\left\{ \exists h : \nu_h(s_h) \neq \pie_h(s_h) \right\}\\
        &+
        \sum_{s_{1:H}}\mathbb{P}^{\nu'}(s_{1:H}) \mathbb{I}\left\{ \exists h : \nu'_h(s_h) \neq \pie_h(s_h) \right\}\\
        =&
        \rho(\nu\parallel \pie) + \rho(\nu' \parallel \pie),
        \label{eq:symmetric_observations_lower}
    \end{aligned}       
\end{equation}
where $s_{1:H} \in \Xcal_2 \cup \Xcal_3 \cup \Xcal_4$ implies either $\nu$ or $\nu'$ disagrees with $\pie$, while the  inequality relaxes the condition by splitting it into separate contributions for $\nu$ and $\nu'$.



We conclude the proof by plugging
\eqref{eq:symmetric_observations_lower}
into \eqref{eq:symmetric_intermediate}.
\end{proof}

\subsection{Proof of Theorem~\ref{thm:log_loss_dagger_step_1}}
\label{sec:dagger_traj_annotation}


 
We first demonstrate that the performance difference between expert and the the uniform first-step mixing of any Markovian policy sequence $\{\pi^{n}\}_{n=1}^N$
is upper bounded by $2R \Ontraj_N/N$, and then show the trajectory-wise sample complexity of Algorithm~\ref{alg:dagger_step_one} in Theorem~\ref{thm:log_loss_dagger_step_1}.


\begin{lemma}
\label{lem:dagger_regret_new_appendix}
For any MDP $\Mcal$, deterministic expert $\pie$, and sequence of policies $\cbr{\pi^n}_{n=1}^N$, each of which can be Markovian or a first-step mixing of Markovian policies, their first step uniform mixture policy $\hat{\pi}$ satisfies.
\[
J(\pie) - J(\hat{\pi}) \leq 2R \cdot \frac{\Ontraj_N}{N} .
\]
\end{lemma}


\begin{proof}
By Lemma~\ref{lem:decoupled_hellinger_bound_general}, for each $\pi^n$, which is  a first-step mixing of Markovian policies:
\[
\EE^{\pi^n} \sbr{D^2_\text{H}(\pi^n(\cdot \| s_{1:H}),\pie(\cdot \| s_{1:H}))}   \geq  \frac{1}{2}\rho(\pi^n \parallel \pie).
\]

Then, by the definition of $\Ontraj_N$,
\[
\frac{\Ontraj_N}{N} 
=
\frac{1}{N}\sum_{n=1}^{N} \EE^{\pi^n} \sbr{D^2_\text{H}(\pi^n(\cdot \| s^n_{1:H}),\pie(\cdot \| s^n_{1:H})) }
\geq
\frac{1}{2N} \sum_{n=1}^{N} \rho(\pi^n \parallel \pie) 
=
\frac{1}{2} \rho(\hat{\pi} \parallel \pie),
\]
where we apply the fact 
that $\hat{\pi}$ is a first-step uniform mixing of $\{\pi^n\}_{n=1}^N$.
Finally, we conclude the proof by applying Lemma~\ref{lem:foster_reduction}.
\end{proof}


%

\begin{proof}[Proof of Theorem~\ref{thm:log_loss_dagger_step_1}]

The proof closely follows Proposition C.2 in~\cite{foster2024behavior},which was tailored to another variant of DAgger. Different from that analysis, here we leverage the distribution of the state sequence \(s_{1:H}\) instead of the per-step state distribution.

Observe that the log loss functions passed through online learning oracle $\alg$, $\ell^n(\pi)$ is of the form 
\[
\ell^n(\pi) 
=
\log\rbr{ 
\frac1 {\pi(a^{n,*}_{1:H}\parallel s^n_{1:H})} } 
\]
It can be observed that $\ell^n(\pi_u)$'s are 1-exp-concave in $u \in \Delta(\Bcal)$. Therefore, setting $\alg$ as the exponential weight  algorithm (Proposition~\ref{prop:log_loss_base}) ensures that the following bound holds almost surely:
\[
\sum_{n=1}^{N} \log(1/\pi^n(a^{*,n}_{1:H}\parallel s^n_{1:H})) \leq \sum_{n=1}^{N}  \log(1/\pie(a^{*,n}_{1:H}\parallel s^n_{1:H})) + \log(B)
=
\log(B).
\]






Now, Lemma~\ref{lem:a14} with $x^n = s^n_{1:H}$, $y^n = a^{*,n}_{1:H}$, $g_* = \pie(\cdot \| \cdot)$, and $\Hcal^n = \{o^{n'}\}_{n'=1}^n$, where $o^n = \left( s_1^n, a_1^n, a_1^{*,n}, \dots, s_H^n, a_H^n, a_H^{*,n} \right)$, 
implies that with probability at least $1-\delta$,
\[
\Ontraj_N
=
\sum_{n=1}^{N} \EE^{\pi^n} \sbr{D^2_\text{H}(\pi^n(\cdot \| s^n_{1:H}),\pie(\cdot \| s^n_{1:H})) } 
\leq
\log(B) + 2\log(1/\delta).
\]
The second part of the theorem follows by applying Lemma~\ref{lem:dagger_regret_new_appendix}.
\end{proof}

\section{Auxiliary Results}
\label{sec:auxiliary_results}

\begin{lemma}
\label{lem:hellinger-pointmass}
If $p, q$ are two distributions over some discrete domain $\Zcal$, and $q$ is {\emph{a delta mass}} on an element in $\Zcal$. Then 
\[
\frac{1}{2} \parallel p - q \parallel_1
\leq 
D_H^2(p, q) 
\leq \parallel p - q \parallel_1
\]
\end{lemma}
{\begin{proof}
Without loss of generality, assume that $q$ is a delta mass on $z_0$. Therefore, 
$\| p - q \|_1 = 2(1 - p(z_0))$, and 
\[
D_H^2(p, q)
= 
(1 - \sqrt{p(z_0)})^2 + (1 - p_{z_0})
= 
2 (1 - \sqrt{ p(z_0) }).
\]
Thus, 
\[
\frac{D_H^2(p, q)}{\| p - q \|_1} = \frac{1}{1 + \sqrt{p(z_0)}} \in [\frac12, 1].
\qedhere
\]
\end{proof}

\begin{lemma}[Performance Difference Lemma~\cite{Kakade2002ApproximatelyOA,ross2010efficient}]
\label{lem: performance difference lemma base}
For two Markovian policies $\pi$ and $\pie$ $: \Scal \to \Delta(\Acal)$, we have 
\[
J(\pie) - J(\pi) = \EE^\pi \sbr{\sum_{h=1}^H A_h^{E}(s_h,a_h)},
\]
where $A_h^E(s_h,a_h) :=Q_h^{\pie}(s_h, a_h) - V_h^{\pie}(s_h) $. Furthermore:


\begin{itemize}
\item It holds that (recall Definition~\ref{def:traj_l1}) 
\[J(\pi) - J(\pi^E) \leq H \cdot \lambda(\pie \parallel \pi).\]

\item Suppose $(\MDP,\pie)$
is $\mu$-recoverable, then 
\[J(\pi) - J(\pi^E) \leq \mu \cdot \lambda(\pi \parallel \pie).\]
\end{itemize}
\end{lemma}




\begin{lemma} [Lemma D.2. of ~\cite{foster2024behavior}] 
For all (potentially stochastic) policies $\pi$ and $\pi'$, it holds that
\[
J(\pi) - J(\pi') \leq R \cdot \rho(\pi\parallel \pi').
\]

\label{lem:foster_reduction}
\end{lemma}

\begin{proposition}[Proposition 3.1 of~\cite{cesa2006prediction}]
Suppose $\{\ell^n(\cdot)\}_{n=1}^N$ is a sequence of $\eta$-exp-concave functions from $\Delta(\Xcal)$ to $\mathbb{R}$. For all $x\in \Xcal$, define the weights $w^{n-1}(x)$ and probabilities $u^n(x)$ as follows:
\[
w^{n-1}(x) = e^{-\eta \sum_{i=1}^{n-1} \ell_i(e_x)}, \quad
u^n(x) = \frac{w^{n-1}(x)}{\sum_{x'\in \Xcal} w^{n-1}(x')},
\]
where $e_x$ is the $x$-th standard basis vector in $\mathbb{R}^{|\Xcal|}$. Then, choosing $u^n = (u^n(x))_{x\in\Xcal}$ (exponential weights used with learning rate $\eta$) satisfies:
\[
\sum_{n=1}^N \ell^n(u^n)
\leq  
\min_{x \in \Xcal}  \sum_{n=1}^N \ell^n(e_x)
+
\frac{\log |\Xcal|}{\eta}.
\]
\label{prop:log_loss_base}
\end{proposition}

\begin{lemma}[Corollary of Lemma A.14 in~\cite{foster2021statistical}]

Under the realizbility assumption, where there exists $g_{\star} \in \mathcal{G}$ such that for all $n \in [N]$,
\[
y^{n} \mid x^{n}, \mathcal{H}^{n-1} \sim g_{\star} (\cdot \mid x^{n}),
\]
where $\mathcal{H}^{n-1}$ denotes all histories at the beginning of round $n$.
Then, for any estimation algorithm that outputs $\cbr{\hat{g}^n}_{n=1}^N$ online
and any $\delta \in (0,1)$, with probability at least $1 - \delta$,
\[
\sum_{n=1}^{N} \mathbb{E}_{n-1} \left[ D_{\mathcal{H}}^2 \left( \hat{g}^{n}( \cdot \mid x^{n}), g_{\star}(\cdot \mid x^{n}) \right) \right]
\leq \sum_{n=1}^{N} \left( 
\log\frac{1}{\hat{g}^n(y^n \mid x^n)}
-
\log\frac{1}{g_{\star}(y^n \mid x^n)}
 \right) + 2 \log (\delta^{-1}).
\]
where 
$\mathbb{E}_n[\cdot] := \mathbb{E}[\cdot \mid \mathcal{H}^{n}]$.
\label{lem:a14}
\end{lemma}

\section{Experiment Details}
\label{sec:appendix_experiment_results}

We compare \WarmStartDagger with Behavior Cloning (BC) and \StateDagger on continuous-control tasks from OpenAI Gym MuJoCo~\cite{todorov2012mujoco, brockman2016openai} with episode length $H=1000$. 

\paragraph{Infrastructure and Implementation.}
All experiments were conducted on a Linux workstation equipped with an Intel Core i9 CPU (3.3GHz) and four NVIDIA GeForce RTX 2080 Ti GPUs. Our implementation builds on the publicly available DRIL repository~\cite{brantley2019disagreement} (https://github.com/xkianteb/dril), with modifications to support interactive learning. The continuous control environments used in our experiments are: ``HalfCheetahBulletEnv-v0'', ``AntBulletEnv-v0'', ``Walker2DBulletEnv-v0'', and ``HopperBulletEnv-v0''. We include link to our implementation here: \url{https://github.com/liyichen1998/Interactive-and-Hybrid-Imitation-Learning-Provably-Beating-Behavior-Cloning}.

\paragraph{Environments and Expert Policies.} We use four MuJoCo environments: Ant, Hopper, HalfCheetah, and Walker2D. The expert policy is a deterministic MLP pretrained via TRPO~\cite{schulman2015trust, schulman2017proximal}, with two hidden layers of size 64.

\paragraph{Model Architecture used by Learner.} The learner uses the same MLP architecture as the expert. Following~\cite{foster2024behavior}, we use a diagonal Gaussian policy:
\[
\pi(a \mid s) = \mathcal{N}\left(f_\theta(s), \operatorname{diag}(\sigma^2)\right),
\]
where $f_\theta(s) \in \mathbb{R}^{d_\Acal}$ is the learned mean, and $\sigma \in \mathbb{R}^{d_\Acal}$ is a learnable log-standard deviation vector. 

Each model is trained from random initialization using a batch size of 100, a learning rate of $10^{-3}$, and up to 2000 passes over the dataset, with early stopping evaluated every 250 passes using a 20\% held-out validation set.

\paragraph{Learning Protocols.}
To evaluate the performance of BC against the number of states annotated, we reveal expert state-action pairs sequentially along expert trajectories until the annotation budget is reached. For \StateDagger, at each iteration, it rolls out the latest policy, samples a state uniformly from the trajectory, queries it for the expert action, and updates immediately.

For \WarmStartDagger, we begin with BC and switch to \StateDagger after a predefined number of offline expert state-action pairs has been used, denoted as $N$. We set $N$ to be 100, 200, or 400 for easier tasks (e.g., Hopper, Ant) and 200, 400, or 800 for harder tasks (e.g., HalfCheetah, Walker2D).

\paragraph{Cost Model and Evaluation.} We assign a cost of 1 to each offline state-action pair and a cost of $C = 1$ or $2$ to each interactive query. We run each method for 10 random seeds. For every 50 new state-action pairs collected, we evaluate the current policy by running 25 full-episode rollouts and reporting the average return.

Though the nonrealizable setting is beyond the scope of this work, we expect that some variant of our algorithm can still give reasonable performance, provided that the policy class is expressive enough (so that
the approximation error is nonzero but small). For example, \cite{li2023agnostic} observed that with nonrealizable stochastic experts, DAgger variants outperform BC, and exhibit learning curves similar to ours.

\subsection{Additional Experiment Plots}

We present extended experiment results with larger cost budgets. As shown in Figure~\ref{fig:log_comparison_appendix}, we allocate a total annotation  cost budget of 2000 for Hopper and Ant, and 4000 for HalfCheetah and Walker.
This complements Figure~\ref{fig:log_comparison_main} in 
the main paper by showing the full training curves without zooming into the stage with small cost budget. The trends are consistent with our earlier observations: \WarmStartDagger achieves similar or better sample efficiency compared to \StateDagger when $C=1$, and clearly outperforms both baselines under the cost-aware setting where $C=2$.


\begin{figure*}[t]  
  \centering
  \includegraphics[width=\textwidth]{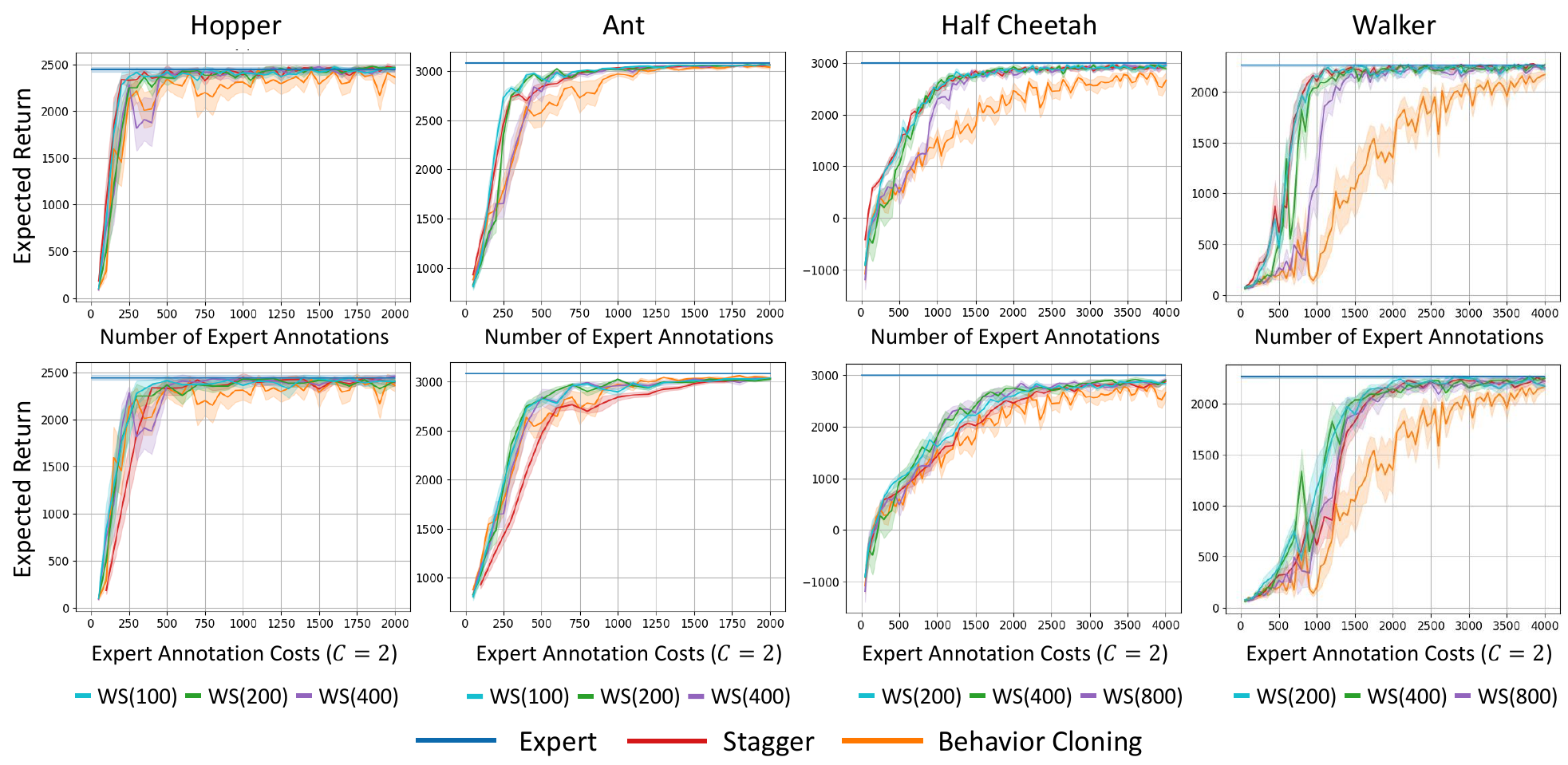}
  \caption{Sample and cost efficiency on MuJoCo tasks. The top row shows expected return vs. number of annotations ($C = 1$); the bottom row shows performance under a cost-aware setting ($C = 2$). \WarmStartDagger (WS) is initialized with 1/20, 1/10, or 1/5 of the samples as offline demonstrations. It matches \StateDagger in sample efficiency and outperforms the baselines when $C = 2$, especially WS(1/5).}
  \label{fig:log_comparison_appendix}
\end{figure*}

\subsection{Experiment with MSE Loss}
\label{sec:mse_experiment}

We additionally evaluate our algorithms using mean squared error (MSE) as the loss function for optimization. All training settings remain identical to the main experiments with log loss, except that we use a learning rate of $2.5 \times 10^{-4}$. As shown in Figure~\ref{fig:mse_comparison_appendix}, we observe qualitatively similar results to those under log loss shown in Figure~\ref{fig:log_comparison_main}, consistent with prior observations in~\cite{foster2024behavior}, with perhaps more stable learning curves.

\begin{figure*}[t]  
  \centering
  \includegraphics[width=\textwidth]{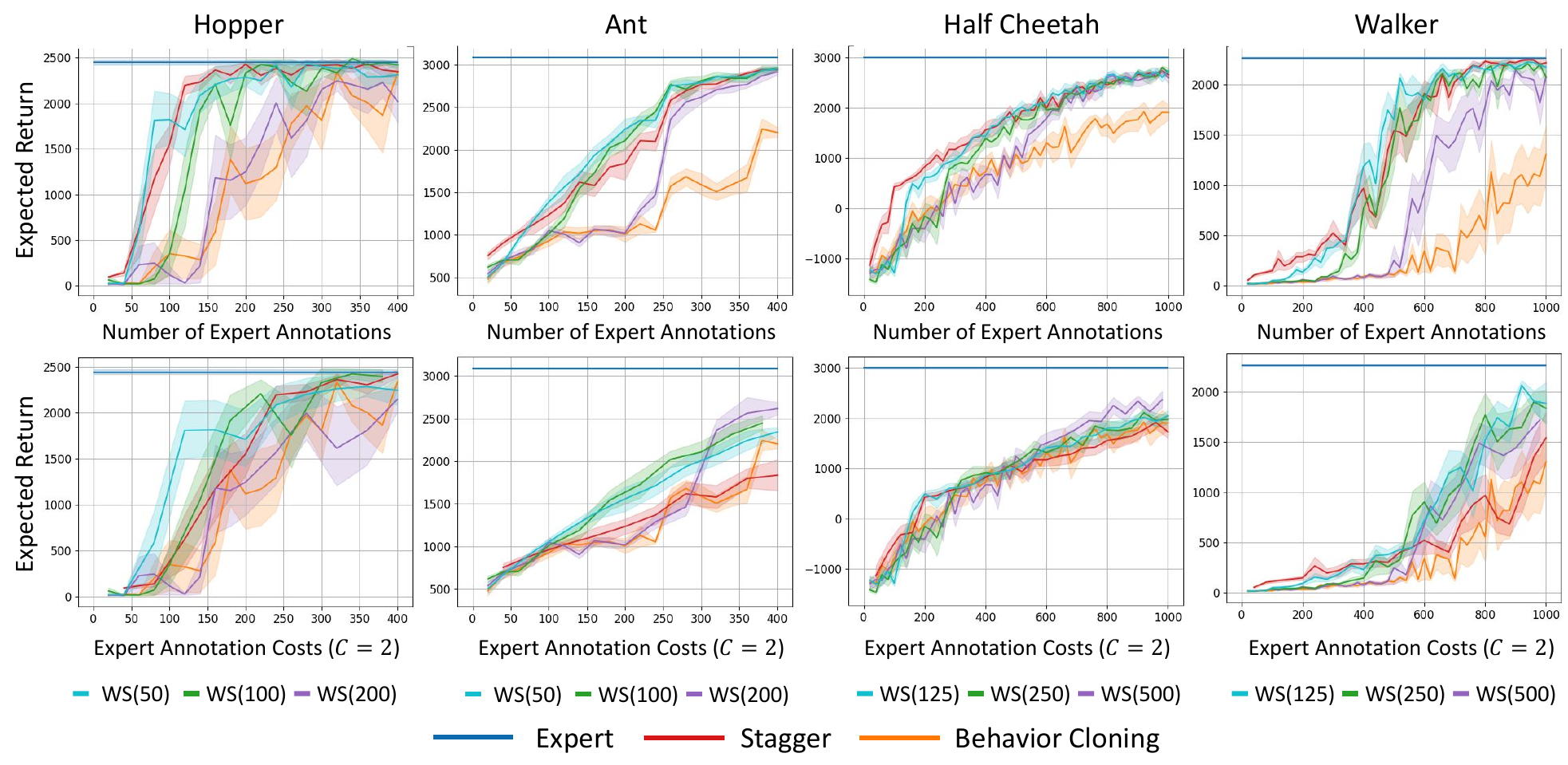}
  \caption{
Performance comparison under MSE loss across MuJoCo tasks. Results show that \WarmStartDagger (WS) achieves comparable sample efficiency and performance to the log loss setting, with improved training stability. Each curve represents the average over 10 seeds.
}
  \label{fig:mse_comparison_appendix}
\end{figure*}

\subsection{Additional Experiments with \TrajDagger}
\label{sec:tragger_experiments_appendix}

For completeness, we evaluate \TrajDagger and its warm-start variant (\WarmStartTragger, Algorithm~\ref{alg:warm_il_traj}), on continuous control tasks and the toy MDP as in Figure~\ref{fig:mdp_design}. The key distinction between \WarmStartTragger and \WarmStartDagger lies in the annotation mode: the former employs trajectory-wise oracle feedback instead of state-wise annotation, leading to notably different behaviors, as shown in Figure~\ref{fig:traj_dagger_overall_appendix}.

In particular, for Ant and HalfCheetah, the state-wise cost efficiency ($C=1$) of \TrajDagger and \WarmStartTragger is significantly worse than that of \StateDagger due to the cold-start problem: early DAgger rollouts have poor state coverage but must still proceed until the end of each trajectory. In contrast, \StateDagger samples only one state per trajectory, thereby better leveraging interaction to get timely feedback.
For Hopper and Walker, however, \TrajDagger and \WarmStartTragger achieve performance closer to \StateDagger. These environments feature hard resets when the agent's state becomes unhealthy (unlikely in Ant and never in HalfCheetah), which terminates ('truncate' has a specific meaning of ending at a prespecified step in openai gym) poor trajectories and consequently improve sample efficiency.

Overall, these observations suggest a natural middle ground between full-trajectory and single-state annotation—namely, batch queries (e.g., sampling 50 states per trajectory), as explored by~\cite{li2023agnostic} with comparable results.

A head-to-head comparison between \TrajDagger and \StateDagger, as well as between \WarmStartTragger and \WarmStartDagger, is shown in Figure~\ref{fig:traj_dagger_hth_appendix}, highlighting the advantage of state-wise over trajectory-wise annotation.

However, this advantage does not hold in general: in the toy MDP in Figure~\ref{fig:mdp_design}, \TrajDagger and \WarmStartTragger achieve performance nearly identical to \StateDagger and \WarmStartDagger, as shown in Figure~\ref{fig:traj_dagger_mdp_appendix}.

\begin{algorithm}[t]
\caption{\WarmStartTragger: Warm-start \TrajDagger with offline demonstrations}
\label{alg:warm_il_traj}
\begin{algorithmic}[1]
\STATE \textbf{Input:} MDP $\Mcal$, trajectory-wise expert annotation oracle $\Otraj$, Stationary policy class $\Bcal$, online learning oracle $\alg$, offline expert dataset $D_{\mathrm{off}}$ of size $\Noff$, interaction budget (in terms of number of states) $\Non$
\STATE Initialize $\alg$ with policy class 
$\Bcal_{\text{bc}} := \{ \pi \in \Bcal : \pi(s_h) = a_h,\ \forall h \in [H], \forall (s, a)_{1:H} \in D_{\text{off}} \}$
\FOR{$n = 1, \dots, \Non/H$}
    \STATE Query $\alg$ and receive $\pi^n$.
    \STATE Execute $\pi^n$ and sample $s^n_{1:H}$ following $\PP^{\pi^n}$. Query $\Otraj$ for $a_{1:H}^{*,n}=\pie(s^n_{1:H})$.
    \STATE Update $\alg$ with loss function
    \begin{equation}
    \ell^n(\pi)
    := \log\rbr{\frac1{
    \pi(a_{1:H}^{*,n} \parallel s_{1:H}^n)}}
    .
    \end{equation}
\ENDFOR
\STATE \textbf{Output:} $\hat{\pi}$, a first-step uniform mixture of $\{\pi^1, \dots, \pi^N\}$.
\end{algorithmic}
\end{algorithm}

\begin{figure*}[t]  
  \centering
  \includegraphics[width=\textwidth]{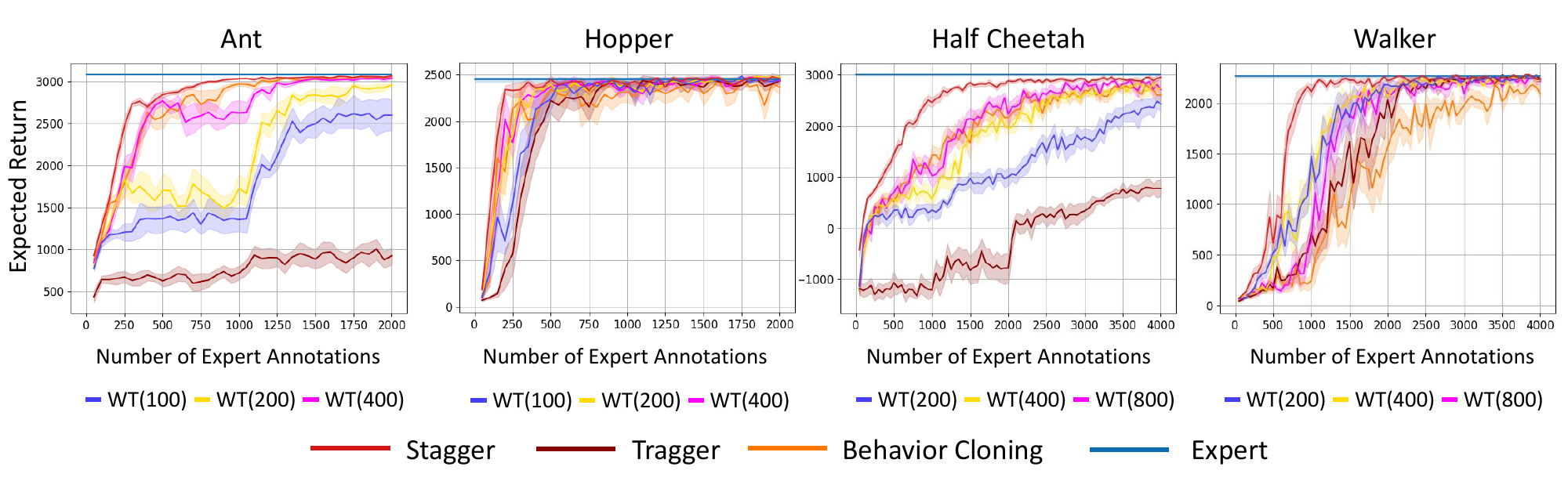}
\caption{Sample efficiency of algorithms on MuJoCo tasks, showing expected return vs. number of annotations ($C = 1$). \WarmStartTragger (WT) is initialized with 1/20, 1/10, or 1/5 of the total annotation budget as offline demonstrations. 
Specifically, WT($n$) refers to WT with offline demonstrations of total length $n$.
Although the performance of WT improves with more offline demonstrations, both \TrajDagger and \WarmStartTragger remain inferior to \StateDagger and, in many cases, even underperform Behavior Cloning, confirming the advantage of state-wise over trajectory-wise annotations.}
  \label{fig:traj_dagger_overall_appendix}
\end{figure*}

\begin{figure*}[t]  
  \centering
  \includegraphics[width=\textwidth]{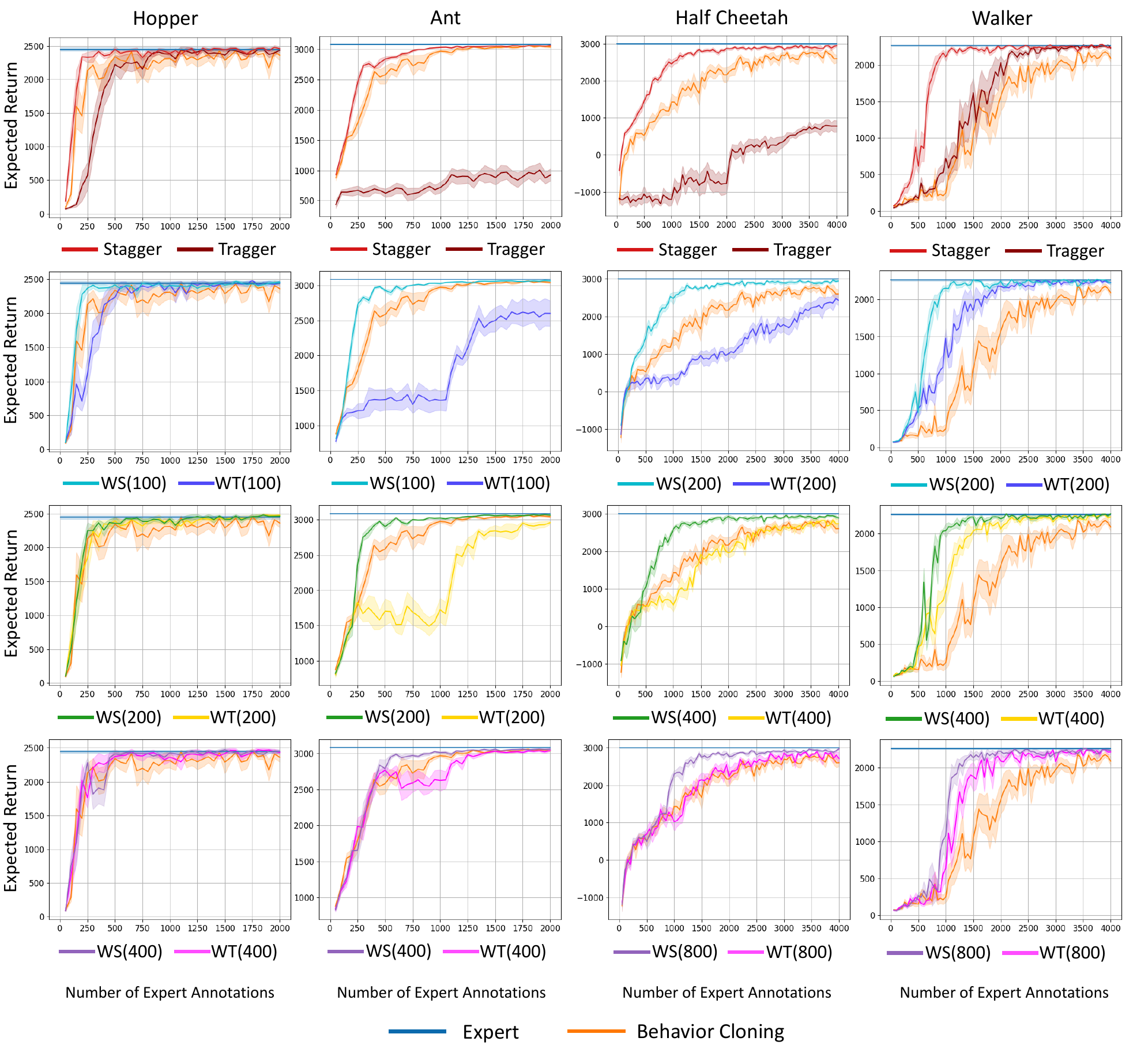}
  \caption{Head-to-head sample efficiency comparison between \TrajDagger and \StateDagger, and between \WarmStartTragger and \WarmStartDagger under equal (since we are talking about comparison here) offline demonstration budgets. \StateDagger and \WarmStartDagger consistently outperform \TrajDagger and \WarmStartTragger. The performance gap narrows as the offline budget increases, effectively alleviating the cold-start problem suffered by \TrajDagger.}
  \label{fig:traj_dagger_hth_appendix}
\end{figure*}

\begin{figure*}[t]  
  \centering
  \includegraphics[width=\textwidth]{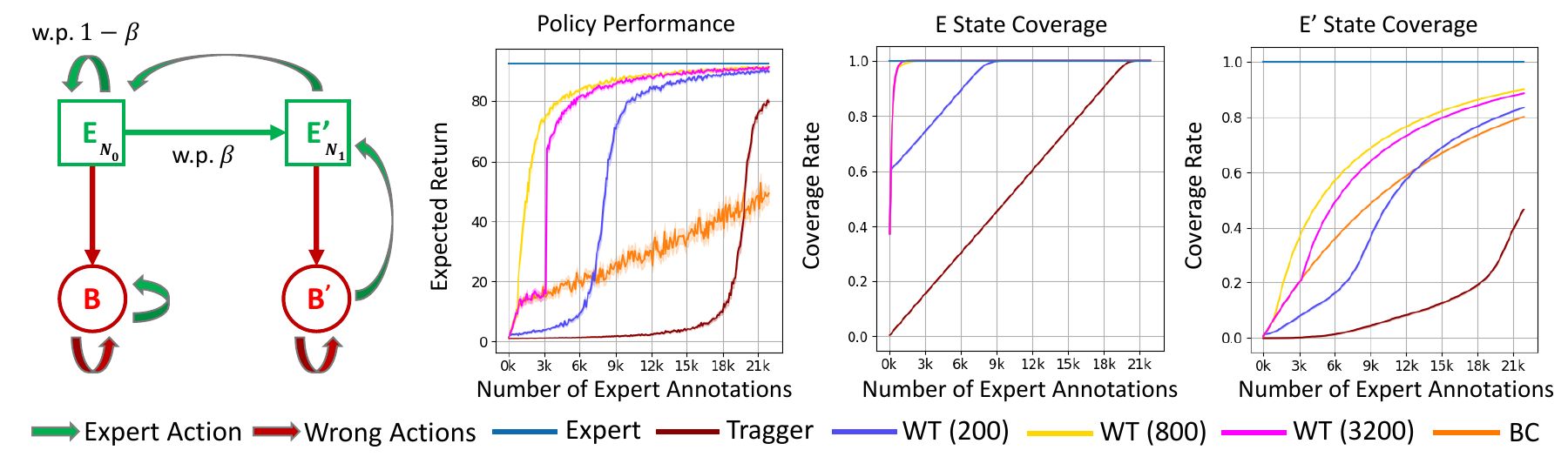}
  \caption{
  Similar to Figure~\ref{fig:mdp_design}, we evaluate 
  \TrajDagger
  and \WarmStartTragger (WT) with 
  200, 800, 3200 offline (state, expert action) pairs in the toy MDP therein. All methods are evaluated under equal total annotation cost with $C = 1$.
  With 800 offline (state, expert action) pairs, WT significantly improves the sample efficiency over the baselines and explores $\mathbf{E'}$ more effectively. The performance of \TrajDagger  and \WarmStartTragger is almost the same as \StateDagger and \WarmStartDagger
  in Figure~\ref{fig:mdp_design}.
  }
  \label{fig:traj_dagger_mdp_appendix}
\end{figure*}

%% file: additional_literature.tex
\paragraph{First-step mixing and each-step mixing policies.}
The emergence of first-step mixing policies originated from technical considerations. In may interactive IL methods~\cite{ross2011reduction, ross2014reinforcement}, the returned policy was not a uniform first-step mixture but rather the best policy selected through validation. However, performing such validation in an interactive setting often requires additional expert annotations.
Subsequent works~\cite{rajaraman2021value,li2022efficient,li2023agnostic,foster2024behavior} circumvented the need for validation by employing a uniform first-step mixture of policies across learning rounds, thereby directly translating online regret guarantees into performance differences. 
Our 
\TrajDagger algorithm (Algorithm~\ref{alg:dagger_step_one} in Appendix~\ref{sec:decoupled_h_dist}) also employs a first-step mixing policy at each iteration, and has state-wise sample complexity on par with behavior cloning.

On the other hand, each-step mixing between the learned policy across rounds and the expert policy has been a prevalent strategy in interactive IL approaches~\cite{daume2009search,ross2010efficient,ross2011reduction,ross2014reinforcement}. 
For each-step mixture policies, \cite{li2022efficient} was the first to explicitly distinguish this approach from first-step mixing. In other works~\cite{rajaraman2021value,foster2024behavior}, each-step mixing can be interpreted as learning $H$ separate mixture policies, one for each step within an episode.

\paragraph{Alternative algorithm designs and practical implementations.}
Though this work follows~\cite{foster2024behavior} and focuses on log loss, we believe our $O(1/N)$ rate is not exclusive to log loss. Despite requiring an additional supervision oracle, \cite{le2018hierarchical} suggests that trajectory-wise annotation complexity similar to Theorem~\ref{thm:log_loss_dagger_main} (and Theorem~\ref{thm:log_loss_dagger_step_1})
can be achieved using Halving~\cite{shalev2011online} and $0$-$1$ loss.

From an algorithmic perspective, we explored trajectory-wise annotation with first-step mixing (Algorithm~\ref{alg:dagger_step_one} in Appendix~\ref{sec:decoupled_h_dist}) and state-wise annotation with each-step mixing (Algorithm~\ref{alg:dagger_one_sample}). 
For trajectory-wise annotation with each-step mixing, naively learning a parameter-sharing policy may encounter a batch-summed log loss, introducing an (undesirable) additional $H$ factor to the sample complexity~\cite{joulani2013online,wan2022online}. Analyzing state-wise annotation with first-step mixing remains an open question.


For practical implementations, it is worth noting that even with oracle-efficient implementations (e.g.~\cite{li2022efficient,li2023agnostic}), interactive IL may require multiple supervised learning oracle calls per iteration. In contrast, offline IL requires only a single oracle call to obtain the returned policy, which provides a clear computational advantage. We also note that real-world experts can be suboptimal; in some applications it may be preferable to combine imitation and reinforcement learning signals (e.g., \cite{ross2014reinforcement,sun2018truncated,amortila2022few}). 

\paragraph{Information-theoretic lower bounds for interactive imitation learning.}
A line of works~\cite{rajaraman2020toward,rajaraman2021value,foster2024behavior} provides lower bounds for the sample complexity of imitation learning under the realizable setting and considers $\mu$-recoverability. \cite{rajaraman2021value} is the first to demonstrate a gap between the lower bounds of offline IL and interactive IL in trajectory-wise annotation, focusing on the tabular and non-parameter-sharing setting. \cite{foster2024behavior} establishes a $\Omega\left(\frac{H}{\epsilon}\right)$ sample-complexity lower bound for trajectory-wise annotation in the parameter-sharing setting.



We observe that the proof of ~\cite[][Theorem 2.2]{foster2024behavior} also implicitly implies a $\Omega(\frac{H}{\epsilon})$ sample complexity lower bound for the state-wise annotation setting. Their proof relies on an MDP consisting only of self-absorbing states, where annotating a full trajectory gives the same amount of information as annotating a single state. In that MDP (which is $1$-recoverable), Algorithm~\ref{alg:dagger_one_sample} achieves $\tilde{O}(\frac{H\log(B)}{\epsilon})$ state-wise sample complexity, which does not contradict this lower bound. 
Nonetheless, obtaining lower bounds for state-wise sample complexity for general MDPs, policy classes, and general recoverability constants remains an open question.